\let\clineorig\cline
\documentclass[sn-mathphys,Numbered]{sn-jnl}% Math and Physical Sciences Reference Style
\setcitestyle{authoryear,round}
\usepackage{graphicx}%
\usepackage{multirow}%
\usepackage{amsmath,amssymb,amsfonts}%
\usepackage{amsthm}%
\usepackage{mathrsfs}%
\usepackage[title]{appendix}%
\usepackage{xcolor}%
\usepackage{textcomp}%
\usepackage{manyfoot}%
\usepackage{booktabs}%
\usepackage{algorithm}%
\usepackage{algorithmicx}%
\usepackage{algpseudocode}%
\usepackage{listings}%
\usepackage{bbm}%
\usepackage{tikz}%
\usepackage{xspace}%
\usepackage{makecell}%
\usepackage{diagbox}%
\usepackage{doi}%
\usepackage{float}%
\usepackage{natbib}%
\def\ie{{\it i.e.,\ \/}}
\def\eg{{\it e.g.,\ \/}}

\def\defeq{{\,\stackrel{\Delta}{=}}\,}

\def\nn{{\nonumber}}
\newcommand{\Tau}{\mathcal{T}}
\newcommand{\ACK}{\mathop{\text{ACK}}}

\newtheorem{theorem}{Theorem}%  meant for continuous numbers
\newtheorem{lemma}{Lemma}

\newtheorem{corollary}{Corollary}

\newtheorem{definition}{Definition}%

\begin{document}

\title[Low-Complexity Algorithm for Restless Bandits with Imperfect Observations]{Low-Complexity Algorithm for Restless Bandits with Imperfect Observations}


\author*[1,2]{\fnm{Keqin} \sur{Liu}}\email{kqliu@nju.edu.cn}

\author[3]{\fnm{Richard} \sur{Weber}}\email{rrw1@cam.ac.uk}

%\author[1]{\fnm{Ting} \sur{Wu}}\email{tingwu@nju.edu.cn}

\author[2]{\fnm{Chengzhong} \sur{Zhang}}\email{171840780@smail.nju.edu.cn}

\affil*[1]{\orgdiv{Department of Mathematics}, \orgname{Nanjing University}, \orgaddress{\city{Nanjing}, \postcode{210093}, \country{China}}}

\affil[2]{\orgname{National Center for Applied Mathematics}, \orgaddress{\city{Nanjing}, \postcode{210093}, \country{China}}}

\affil[3]{\orgdiv{Department of Mathematics}, \orgname{University of Cambridge}, \orgaddress{\city{Cambridge}, \postcode{CB3 0WB}, \country{UK}}}


\abstract{We consider a class of restless bandit problems that finds a broad
application area in reinforcement learning and stochastic optimization. We consider $N$ independent discrete-time
Markov processes, each of which had two possible states: 1 and 0
(`good' and `bad'). Only if a process is both in state 1 and
observed to be so does reward accrue. The aim is to maximize the
expected discounted sum of returns over the infinite horizon subject
to a constraint that only $M$ $(<N)$ processes may be observed at
each step. Observation is error-prone: there are known probabilities
that state 1 (0) will be observed as 0 (1). % Enter your abstract
From this one knows, at any time $t$, a probability that process $i$ is
in state 1.
%Due to the partial and imperfect observation
The resulting system may be modeled as a restless multi-armed
bandit problem with an information state space of uncountable
cardinality.
% , i.e. the $N$ probabilities that processes are in state  1.
Restless bandit problems with even finite state spaces are
PSPACE-HARD in general. We propose a novel approach
for simplifying the dynamic programming equations of this class of
restless bandits and develop a low-complexity algorithm that
achieves a strong performance and is readily extensible to the
general restless bandit model with observation errors. Under certain
conditions, we establish the existence (indexability) of Whittle
index and its equivalence to our algorithm. When those conditions do
not hold, we show by numerical experiments the near-optimal
performance of our algorithm in the general parametric space. Furthermore,
we theoretically prove the optimality of our algorithm for
homogeneous systems.}




\maketitle

\section{Introduction}

The exploration-exploitation (EE) dilemma is well posed in
optimization-over-time problems and mathematically modeled in various
forms for reinforcement learning, to which a major category,
\emph{multi-armed bandits} (MAB) belongs. In the classical MAB model, a
player chooses one out of $N$ statistically independent arms to pull
and possibly accrues reward determined by the state of the chosen arm
which transits to a new state according to a known Markovian
rule~\citep{GGW2011}. The states of other arms remain frozen. The
objective is to maximize the expected total discounted reward summed
over times $t=1,2,\ldots\ $ to an infinite time horizon with discount
factor $\beta\in(0,1)$,
\begin{eqnarray}
	\max_{\pi\in\Pi}\mathbbm{E}_{\pi}\left[\sum_{t=1}^{\infty}\beta^{t-1}R(t)\Big\vert
	S(t),A(t)\right].
\end{eqnarray}
The expectation is taken under some policy $\pi$, chosen from the set
of all feasible policies $\Pi$; $S(t)$ is the joint state of all arms
at time $t$, $A(t)\in\{1,2,\ldots,N\}$ is the arm pulled at $t$ and
$R(t)$ is the reward thus obtained. It follows from standard
theory of Markov decision processes that there must exist an optimal
stationary policy~$\pi^*$, independent of time~$t$. If each arm's
state space has cardinality $K$, then the joint state space has size
$K^N$. This means that a dynamic programming solution to the problem
will have running time that grows geometrically as the number of arms
increases. \citet{G1979} solved the problem by showing the optimality of an \emph{index} policy, %
\ie for each state of each arm there exists an index depending solely on
the parameters of that arm; it is then optimal at each time to choose
the arm whose state has highest index. The running time of the Gittins
index policy grows only linearly with the number of arms as they are
decoupled when computing the index function(of the states of each
arm). \citet{W1988} generalized Gittins index to the \emph{restless}
MAB model in which those arms that are not chosen may also change
states and produce reward. Whittle's generalization has been shown to
perform very well in theoretical and numerical studies (see,
\eg\citet{WW1990,WW1991,LZ2010,HF2017,ZJW2019,BS2020,LZ2021,GGY2021}). In general, however,
it is difficult to theoretically establish the condition that is necessary for the
Whittle index to exist (so called \emph{indexability}) and to solve for the closed-form Whittle index when it does exists due to the curse of dimensionality (see Sec.~\ref{sec:modelBelief}). For finite-state models, \citet{GGK2023} proposed an efficient algorithm to numerically test indexability and compute Whittle index. A natural question to ask if one can transform a bandit with a continuous-state space into a finite-state one by discretization. Unfortunately, we found that how fine the discretization needs to be given the system paramters is itself a difficult problem. Furthermore, a finer discretization inevitably leads to a larger state transition matrix with increased algorithmic complexity and unpredictability of the steady-state system performance. This motivates us to consider alternative approaches to deal with bandits with infinite state spaces as addressed in this paper. In terms of searching for the general optimal policy, \citet{PT1999} have shown that
the restless MAB with a finite state space is PSPACE-HARD in general. With an infinite state space, a restless bandit problem yields practical difficulties in implementing purely numerical methods as discussed above. In this paper, we show that for our particular Markovian model with an infinite state space, Whittle index policy can be efficiently implemented with satisfactory accuracy after theoretical analysis on the rich mathematical structure of the problem.

In this paper, we extend the work in~\citet{LZ2010} (for a perfect
observation model) and~\citet{LZK2010} (for the myopic policy on
stochastically identical arms) to build a near-optimal algorithm with
low complexity for a class of restless bandits with an infinite state
space and an imperfect observation model. Our model also belongs to
the general framework of partially observable Markov decision
processes (POMDP)~\citep{S1978}. Consider~$N$ processes each of which
evolves on a $2$-state Markov chain whose state is observed if and
only if the process is chosen. Furthermore, the observation is
\emph{error-prone}: state~$1$ may be observed as~$0$ and vice
versa. Each process is referred to as an arm. At time~$t$, the player
obtains reward of amount~$B_n$ if and only if arm~$n$ is currently
chosen and accurately observed in state~$1$. Under resource
constraints, the player's goal is to select $M~(M<N)$ arms at each
time and maximize the long-term reward. By formulating \emph{the
	belief vector} as the system state for decision making, we show that
the indexability is satisfied under certain conditions. Furthermore,
we propose an efficient algorithm to compute an approximate Whittle
index that achieves a near-optimal performance in general, even if the
conditions for indexability do not hold.

\citet{meshram2018whittle},  \citet{mehta2018rested} and \citet{kaza2019sequential} considered similar models (except for some nuances) and established indexability under much stricter conditions on the system parameters. For example, all the three papers require that the Markov transition probabilities of each arm have differences bounded by~$1/3$ while we do not need any restriction on the transition probabilities. Furthermore, the three papers require that the discount factor~$\beta$ is less than~$1/3$ while our condition on~$\beta$ is a more relaxed closed-form expression of the system parameters. In terms of the computation of the Whittle index, their algorithm is a direct application of general reinforcement learning while our algorithmic framework is based on the detailed analysis of the value functions with a quick convergence to the exact Whittle index function. In this paper, we also plot the performance of the optimal policy to demonstrate the near-optimality of our algorithm in addition to the comparison with the myopic policy. For homogeneous systems (stochastically identical arms), we show that our algorithm is equivalent to the myopic policy and theoretically prove its optimality under certain conditions. \citet{WCYW2018} also considered our model and assumed the optimality of the threshold policy for a single arm while using a very coarse linear approximation to compute the Whittle index function (the key step (a) for the second equality in the proof of Lemma~6 in \citet{WCYW2018} is incorrect). In this paper, we rigorously prove the optimality of the threshold policy for a single arm and establish the indexability under certain conditions and subsequently construct an efficient algorithm for computing the Whittle index function with arbitrary precision with its optimality numerically verified in general and formally proved for a class of homogeneous systems.

The rest of this paper is organized as follows: Sec.~\ref{sec:main} presents our problem formulation and main results on the optimal threshold policy and indexability. Sec.~\ref{sec:WhittleAlgor} presents our algorithm with design details. Sec.~\ref{sec:optimality} presents a theoretic proof of the optimality for homogeneous arms. Sec.~\ref{sec:numericalCon} concludes this paper and provides some future research directions in this field.

\section{Main Results}\label{sec:main}

Consider a restless MAB having $N$ internal $2$-state Markov chains
(arms) of potentially \emph{different} transition probabilities. At
each time~$t$, the player chooses to observe the states of $M$ $(<N)$
arms. Let $S\in\{0~\text{(bad)},1~\text{(good)}\}$ denote the current
state of an arm and let $O$ denote its observation outcome (detection
outcome). The error probabilities are ~$\delta=\Pr(O=1\mid S=0)$
and~$\epsilon=\Pr(O=0\mid S=1)$, \ie the probabilities of miss
detection and false alarm, respectively, in the observation model.
In~\citet{Levy2008Book}, it was shown that the error probabilities~$\delta$ and~$\epsilon$ follow the curve of receiver operating characteristics
(ROC) under the optimal detector that makes $1-\delta$ a concave increasing function from~$0$ to~$1$ over~$\epsilon$.
This matches the intuition that making a detector more sensitive will reduce~$\delta$ but increase~$\epsilon$. Since the optimal detector design is a solved problem and not the focus of this paper, we simply assume that~$\delta$ and~$\epsilon$ are given. If arm $n$ in state $S=1$ is observed in state~$1$
(\ie $S=O=1$), then the player accrues $B_n$ units of reward
from this arm. One of many application examples of this observation
model is to {\em cognitive radios}, where a secondary user aims to
utilize a frequency band (channel/arm) currently unused by the primary
users. Due to energy and policy constraints on the sensor of the
secondary user, only a subset of channels can be sensed at each time
and if any of them is sensed idle ($O=1$), the user can send certain
packets over it to its receiver and obtain an $\ACK$ (acknowledgement)
in the end of the time slot if the channel is indeed idle ($S=1$);
otherwise no $\ACK$ from this channel would be received. Then the
reward~$B_n$ is just the bandwidth of channel~$n$. Clearly, the hard
constraint here should be on the miss detection probability~$\delta$
to guarantee the satisfaction of the primary users, \ie the
disturbance (when a secondary user senses a busy channel as idle and
subsequently sends data over it) to the primary users should be
capped.

\subsection{System Model and Belief Vector}\label{sec:modelBelief}
At each discrete time~$t$, the internal state ($0/1$) of an arm cannot
be observed before deciding whether or not to observe the
arm. Therefore, we cannot use the states of the Markov chains as the
system state for decision making. Applying the general POMDP theory to
our model the \emph{belief state vector} consisting of probabilities
that arms are in state~$1$ given all past observations is a sufficient
statistics for making future decisions~\citep{S1978}:
\begin{align}
	\boldsymbol{\omega}(t)&= (\omega_1(t),\omega_2(t),\cdots,\omega_N(t)),\label{eqn:beliefVector}\\
	\omega_n(t)&=\Pr(S_n(t)=1\mid\text{past observations on arm}~n),~~\forall \ n\in\{1,\cdots,N\},
\end{align}
where~$\omega_n(t)$ is the \emph{belief state} of arm~$n$ at time~$t$
and~$S_n(t)$ its internal state. According to the Bayes' rule, the
belief state (of any arm) itself evolves as a Markov chain with an
infinite state space:%
{
	\begin{align} \omega_{n}(t+1)&=
		\begin{cases}
			p_{11}^{(n)}, & n \in A(t), \ACK_n(t)~(S_n(t)=1, O_n(t)=1) \\[1pt]
			\Tau_n\left(\frac{\epsilon\omega_{n}(t)}{\epsilon\omega_{n}(t)+1-\omega_{n}(t)}\right), & n \in A(t), \text{no} \ACK_n(t) \\[1pt]
			\Tau_n(\omega_{n}(t)), & n \notin A(t)
		\end{cases},\label{eqn:beliefUpdate}
	\end{align}
	\begin{align}
		\Tau _n(\omega_{n}(t))&=\omega_{n}(t)
		p^{(n)}_{11}+(1-\omega_{n}(t))p^{(n)}_{01},\label{eqn:1Tau}~~~~~~~~~~~~~~~~~~~~~~~~~~~~~~~~~~~~~~~
\end{align}}%
where $A(t)\subset\{1,2,\ldots,N\}$ is the set of arms chosen at
time $t$ with $\vert A(t)\vert=M$, $S_n(t)$ and $O_n(t)$ are respectively the
state and observation from arm $n$ at time $t$ if $n\in A(t)$,
$\ACK_n(t)$ the acknowledgement of successful utilization of arm~$n$
for slot~$t$, $\Tau_n(\cdot)$ the one-step belief update operator
without observation, and
$\textbf{P}^{(n)}=\{p^{(n)}_{ij}, {i,j\in\{0,1\}}\}$ the transition
matrix of the internal Markov chain of arm~$n$. Note that when $\epsilon=0$, all three expressions of the next belief state in~\eqref{eqn:beliefUpdate} become linear functions and the problem is reduced to that in \cite{LZ2010} for perfect observations where the value functions for dynamic programming can be directly solved in closed-form. Later, we will see that when $\epsilon\neq0$, the value functions are very difficult to analyze except for a few general properties. Without a fine-grain analysis on the value functions, the indexability and Whittle index cannot be established for our model. Our approach is to start with a finite time horizon and then utilize the backward induction (on time horizon) methodology until we take limits of corresponding functions as the time horizon goes to infinity. The key idea is to obtain analytic bounds on the value functions and their derivatives as functions of the system parameters instead of just numerical computations of these functions. We will show that these bounds, together with some detailed properties of the value functions, are sufficient for our purpose. From~\eqref{eqn:1Tau}, the
$k$-step belief update of an unobserved arm for $k$ consecutive slots
starting from any belief state~$\omega$ is
\begin{gather}\label{eqn:kTau} \Tau_n^{k}(\omega)=\frac{p^{(n)}_{01}-(p^{(n)}_{11}-p^{(n)}_{01})^{k}(p^{(n)}_{01}-(1+p^{(n)}_{01}-p^{(n)}_{11})\omega)}{1+p^{(n)}_{01}-p^{(n)}_{11}}.
\end{gather}
For simplicity of notations, we denote~$\Tau_n^1(\cdot)$ by~$\Tau_n(\cdot)$.
\\ \hspace*{0.8cm}At time~$t=1$, the initial belief state~$\omega_n(1)$ of arm~$n$ can
be set as the stationary distribution~$\omega_{n,o}$ of the internal
Markov chain\footnotemark\footnotetext{Here we assume the internal
	Markov chain with transition matrix~$\textbf{P}^{(n)}$ is
	irreducible and aperiodic.}:
\begin{gather}
	\omega_n(1) = \omega_{n,o} =\lim_{k\rightarrow\infty}\Tau_n^k(\omega')=\frac{p^{(n)}_{01}}{p^{(n)}_{01}+p^{(n)}_{10}},\label{eq:stationary}
\end{gather}
where~$\omega_{n,o}$ is the unique solution to~$\Tau_n(\omega)=\omega$
and~$\omega'\in[0,1]$ an arbitrary probability. Given the initial
belief vector
$\boldsymbol{\omega}(1)=(\omega_1(1),\omega_2(1),\ldots,\omega_N(1))$, we arrive at
the following \emph{constrained} optimization problem:
\begin{eqnarray}
	&\max\limits_{\pi:\boldsymbol{\omega}(t)\rightarrow
		A(t)}\mathbbm{E}_{\pi}\left[\sum\limits_{t=1}^{\infty}\beta^{t-1}R(t)\Big\vert \boldsymbol{\omega}(1)\right],\label{objectStrict}\\
	&\quad\mbox{subject to}\quad\vert A(t)\vert=M,\quad t\ge1.\label{constraintStrict}
\end{eqnarray}
Now the decision problem has a countable state space as modelled by the belief vector for a fixed initial~$\boldsymbol{\omega}(1)$ and an uncountable state space for an arbitrarily chosen~$\boldsymbol{\omega}(1)$. It is clear that fixing~$\boldsymbol{\omega}(1)$, the action-dependent belief
vector~$\boldsymbol{\omega}(t)$ takes possible values growing geometrically with
time~$t$, leading to a high-complexity in solving the problem; this is
the so-called \emph{curse of dimensionality}. In the following, we
adopt Whittle's original idea of Lagrangian relaxation to decouple
arms for an index policy and show some crucial properties of the value
functions of a single arm.

\subsection{Arm Decoupling by Lagrangian Relaxation}
\label{sec:singleArm}
\begin{align}
	\max\limits_{\pi:\boldsymbol{\omega}(t)\rightarrow
		A(t)}&\mathbbm{E}_{\pi}\left[\sum_{t=1}^{\infty}\beta^{t-1}\sum_{n=1}^N\mathbbm{1}_{(n\in
		A(t))}\cdot S_n(t)\cdot O_n(t)\cdot B_n\Big\vert \boldsymbol{\omega}(1)\right]\label{objectRelax}\\
	\mbox{subject to\ }&\mathbbm{E}_{\pi}\left[\sum_{t=1}^{\infty}\beta^{t-1}\sum_{n=1}^N\mathbbm{1}_{(n\notin A(t))}\Big\vert\boldsymbol{\omega}(1)\right]=\frac{N-M}{1-\beta}.\label{constraintRelax}
\end{align}
Clearly constraint~\eqref{constraintRelax} is a relaxation on the
player's action~$A(t)$ from~\eqref{constraintStrict}. Applying the
Lagrangian multiplier~$\mu$ to constraint~\eqref{constraintRelax},
we arrive at the following {\em unconstrained} optimization problem:

\begin{gather}
	\max\limits_{\pi:\boldsymbol{\omega}(t)\rightarrow
		A(t)}\mathbbm{E}_{\pi}\left[\sum_{t=1}^{\infty}\beta^{t-1}\sum_{n=1}^N\left[\mathbbm{1}_{(n\in
		A(t))}S_n(t)O_n(t)B_n+\mu\cdot\mathbbm{1}_{(n\notin
		A(t))}\right]\Big\vert\boldsymbol{\omega}(1)\right].
	\label{max: decouple}
\end{gather}
\normalsize
Fixing~$\mu$, the above optimization is equivalent to~$N$
\emph{independent} unconstraint optimization problem as shown below:
for each $n\in\{1,2,\ldots,N\}$,

\begin{gather}
	\max\limits_{\pi:\omega_n(t)\rightarrow \{0,1\}}\mathbbm{E}_{\pi}\left[\sum_{t=1}^{\infty}\beta^{t-1}\left[\mathbbm{1}_{(n\in A(t))}S_n(t)O_n(t)B_n+\mu\cdot\mathbbm{1}_{(n\notin A(t))}\right]\Big\vert\omega_n(1)\right].
	\label{max: single}
\end{gather}\normalsize
Here~$\pi$ is a single-arm policy that maps the belief state of the
arm to the binary action~$u=1$ (chosen/activated) or~$u=0$
(unchosen/made passive). It is thus sufficient to consider a single
arm for solving problem~\eqref{max: decouple}. For simplicity, we will
drop the subscript~$n$ in consideration of a single-armed bandit
problem without loss of generality. Let~$V_{\beta,m}(\omega)$ denote
the value of~\eqref{max: single} with~$\mu=m$
and~$\omega_n(1)=\omega$, it is straightforward to write out the
dynamic equation of the single-armed bandit problem as follows:
\begin{equation}\label{eq: long time value function}
	V_{\beta,m}(\omega) = \max\{V_{\beta,m}(\omega;u=1);V_{\beta,m}(\omega;u=0)\},
\end{equation}
where~$V_{\beta,m}(\omega;u=1)$ and~$V_{\beta,m}(\omega;u=0)$ denote,
respectively, the maximum expected total discounted reward that can be
obtained if the arm is activated or made passive at the current belief
state~$\omega$, followed by an \emph{optimal} policy in subsequent
slots. Since we consider the infinite-horizon problem, a stationary
optimal policy can be chosen and the time index~$t$ is not needed
in~\eqref{eq: long time value function}. Define the nonlinear
operator~$\phi(\cdot)$ as
\[
\phi(\omega)=\frac{\epsilon\omega}{\epsilon\omega+1-\omega}.
\]
It is easy to see that~$\Tau\circ\phi(\cdot)$ is Lipschitz continuous
on~$[0,1]$:
{\small\begin{align}
		\Big\vert\Tau\left(\tfrac{\epsilon\omega}{\epsilon\omega+1-\omega}\right)-\Tau\left(\tfrac{\epsilon\omega'}{\epsilon\omega'+1-\omega'}\right)\Big\vert&=\Big\vert\tfrac{p_{11}\epsilon\omega+(1-\omega)p_{01}}{\epsilon\omega+1-\omega}-\tfrac{p_{11}\epsilon\omega'+(1-\omega')p_{01}}{\epsilon\omega'+1-\omega'}\Big\vert\label{eqn:TauPhiLC}\\[4pt]
		&=\Big\vert\tfrac{\epsilon(p_{11}-p_{01})(\omega-\omega')}{(1-(1-\epsilon)\omega)(1-(1-\epsilon)\omega')}\Big\vert\leq\frac{\vert p_{11}-p_{01}\vert}{\epsilon}\vert\omega-\omega'\vert.
\end{align}}
We assume that~$\epsilon\neq0$ (otherwise the problem is reduced to
that considered in~\citealp{LZ2010}) and~$p_{11}\neq p_{01}$ (otherwise
the belief update is independent of observations or actions and the
problem becomes trivial). Without loss of generality, set $B=1$. We
have
\begin{gather}
	\left\{
	\begin{array}{ll}
		V_{\beta,m}(\omega;u=1)=(1-\epsilon)\omega+\beta\big[(1-\epsilon)\omega V_{\beta,m}(p_{11})\label{eqn:vf01}\\[5pt]
		\quad\quad\quad\quad\quad\quad\quad\quad+(1-(1-\epsilon)\omega)V_{\beta,m}(\Tau(\phi(\omega)))\big], \\[5pt]
		V_{\beta,m}(\omega;u=0)=m+\beta V_{\beta,m}(\Tau(\omega)).
	\end{array}
	\right.
\end{gather}
Define \emph{passive set}~$P(m)$ as the set of all belief states
in which taking the passive action~$u=0$ is optimal:
\begin{eqnarray}
	P(m)\defeq \{\omega:~V_{\beta,m}(\omega;u=1)\le V_{\beta,m}(\omega;u=0)\}.\label{def:passive}
\end{eqnarray}
Notice that the immediate reward under the active action cannot exceed~$1$ so it is optimal to always make the arm passive if the immediate reward under the passive action exceeds~$1$ ($m\ge1$). This is because we would obtain the maximum immediate reward equal to~$m$ at each time step regardless of the state transitions in this case. On the other hand, if $m<-1/{(1-\beta)}$ then the optimal action must be to activate the arm. To see this, note that the total discounted reward by any policy consists of two parts: the reward under the active actions and the reward under the passive actions. If the optimal policy is to make the arm passive now when $m<-1/{(1-\beta)}$, then the reward obtained under the future active actions by this policy must be greater than $1/{(1-\beta)}$ otherwise the total discounted reward would be negative, contradicting the optimality of the policy since any policy that always activates the arm achieves a nonnegative total discounted reward so should the optimal policy. This is again a contradiction since the total discounted reward under the active actions is upper bounded by~$1/{(1-\beta)}$. Consequently, the passive set $P(m)$ changes from the empty set to
the closed interval $[0,1]$ as~$m$ increases from~$-\infty$ to
$\infty$. However, such change may not be monotonic as~$m$
increases. But if $P(m)$ does increase monotonically with~$m$, then
for each value~$\omega$ of the belief state, one can define the unique
$m$ that makes it join~$P(m)$ and stay in the set
forever. Intuitively, such~$m$ measures in a well-ordered manner the
attractiveness of activating the arm in belief state~$\omega$ compared
to other belief states: the larger is the~$m$ that is required for it
to be passive, the more is the incentive to activate the arm in belief
state~$\omega$, even in the problem without~$m$. This Lagrangian
multiplier~$m$ is thus called `subsidy for passivity' by Whittle who
formalized the following definition of {\em indexability} and {\em Whittle index}~\citep{W1988}.
\begin{definition}\label{def:indexability}
	A restless multi-armed bandit is {\em indexable} if for each
	single-armed bandit in a problem with subsidy $m$ for passivity, the
	set of arm states $P(m)$ in which passivity is optimal increases
	monotonically from the empty set to the whole state space as $m$
	increases from $-\infty$ to $+\infty$. Under indexability, the
	\emph{Whittle index} of an arm state is defined as the infimum
	subsidy $m$ such that the state remains in the passive set.
\end{definition}
For our model in which the arm state is given by the belief vector, the indexability is equivalent to the following:
\begin{eqnarray}
	\mbox{If}~~~V_{\beta,m}(\omega;u=1)\le
	V_{\beta,m}(\omega;u=0),~\mbox{then}~  \nonumber \\
	\forall~m'>m,~V_{\beta,m'}(\omega;u=1)\le V_{\beta,m'}(\omega;u=0).
\end{eqnarray}
Under indexability, the Whittle index~$W(\omega)$ of arm state~$\omega$ is defined as
\begin{eqnarray}
	W(\omega)\defeq \inf\{m:~V_{\beta,m}(\omega;u=1)\le V_{\beta,m}(\omega;u=0)\}.\label{def:whittleIdx}
\end{eqnarray}
In the following we derive useful properties of the value functions
$V_{\beta,m}(\omega;u=1)$, $\ V_{\beta,m}(\omega;u=0)$ and
$V_{\beta,m}(\omega)$. Our strategy is to first establish those
properties for finite horizons and then extend them to the infinite
horizon by the uniform convergence of the value functions of the
former to the latter. Define the $T$-horizon value function
$V_{1,T,\beta,m}(\omega)$ as the maximum expected total discounted
reward achievable over the next $T$ time slots starting from the
initial belief state $\omega$. Then
\begin{eqnarray}
	V_{1,T,\beta,m}(\omega) = \max\{V_{1,T,\beta,m}(\omega;u=1);V_{1,T,\beta,m}(\omega;u=0)\},
\end{eqnarray}
where~$V_{1,T,\beta,m}(\omega;u=1)$ and~$V_{1,T,\beta,m}(\omega;u=0)$
denote, respectively, the maximum expected total discounted reward
achievable given the initial active and passive actions over the
next~$T$ time slots starting from the initial belief state~$\omega$:

\begin{align}
	V_{1,T,\beta,m}(\omega;u=1)&=(1-\epsilon)\omega+(1-\epsilon)\omega\beta
	V_{1,T-1,\beta,m}(p_{11}) \nonumber\\
	&~~~~+(1-(1-\epsilon)\omega)
	\beta V_{1,T-1,\beta,m}\left(\Tau(\tfrac{\epsilon\omega}{1-(1-\epsilon)\omega})\right),\label{eqn:finite1} \\
	V_{1,T,\beta,m}(\omega;u=0)&=m+\beta V_{1,T-1,\beta,m}(\Tau(\omega)),\label{eqn:finite0}~~~~~~~~~~~~~~~\\
	V_{1,0,\beta,m}(\cdot)&\equiv 0.~~~~~~~~~~~~~~~~~~~~~~~~~~~~~~~~~~~
\end{align}
From the above recursive equations, we can analyze $V_{1,T,\beta,m}(\omega)$ by backward induction on~$T$. It is easy to see that for any~$\omega$,
\begin{eqnarray}
	V_{1,1,\beta,m}(\omega;u=1)=(1-\epsilon)\omega,\quad
	V_{1,1,\beta,m}(\omega;u=0)=m.\label{eqn:lastStep}
\end{eqnarray}
Therefore~$V_{1,T,\beta,m}(\omega)$ is the maximum of two linear equations and thus piecewise linear and convex for~$T=1$ (in both~$\omega$ and~$m$). Assume that~$V_{1,T-1,\beta,m}(\omega)$ is piecewise linear and convex. The Bayes' rule shows that the following term
\begin{eqnarray}
	(1-(1-\epsilon)\omega)\beta V_{1,T-1,\beta,m}\left(\Tau(\tfrac{\epsilon\omega}{1-(1-\epsilon)\omega})\right)
\end{eqnarray}
is piecewise linear and convex since the leading coefficient $(1-(1-\epsilon)\omega)$ also appears as the
denominator of the argument of the linear operator~$\Tau$ in $V_{1,T-1,\beta,m}(\cdot)$ assumed to be piecewise linear and convex by the induction hypothesis. Henceforth, the
recursive equation set~\eqref{eqn:finite1} and~\eqref{eqn:finite0}
shows that~$V_{1,T,\beta,m}(\omega)$ is the maximum of two convex and
piecewise linear functions and thus piecewise linear and convex for
any~$T>1$ (in both~$\omega$ and~$m$). Motivated by the Lipschitz
continuity of~$\Tau\circ\phi$, we show in Lemma~\ref{lm:LipCont} that
$V_{1,T,\beta,m}(\omega)$ is also Lipschitz continuous under certain
conditions. In the following, we first establish a monotonic property
of $V_{1,T,\beta,m}(\omega)$ in the case of $p_{11}>p_{01}$
(positively correlated Markov chain).
\begin{lemma}\label{lm:posVd}
	If $p_{11}> p_{01}$, then $V_{1,T,\beta,m}(\omega)$ is monotonically increasing with~$\omega\in[0,1]$ for any $T\ge1$.
\end{lemma}
\begin{proof}
	Since $V_{1,T,\beta,m}(\omega)$ is piecewise linear, it is
	differentiable almost everywhere except on a null set (under the
	Lebesgue measure on $\mathbbm{R}$) consisting of finite points among
	which both the left and right derivatives at any point exist but not
	equal. To prove that the continuous function
	$V_{1,T,\beta,m}(\omega)$ is monotonically increasing with~$\omega$,
	we only need to show
	\begin{eqnarray}
		V'_{1,T,\beta,m}(\omega)\ge 0,\quad\forall\omega\in(0,1), \label{eqn:posDiff}
	\end{eqnarray}
	where $V'_{1,T,\beta,m}(\omega)$ denotes the \emph{right} derivative
	of $V_{1,T,\beta,m}(\cdot)$ as a function of the belief state with~$m$
	fixed. From~\eqref{eqn:lastStep}, the value function
	$V_{1,1,\beta,m}(\omega)=\max\{(1-\epsilon)\omega,m\}$ is
	monotonically increasing with nonnegative right derivative
	$1-\epsilon$ or $0$. Assume~\eqref{eqn:posDiff} is true for $T\ge1$,
	then for $T+1$ we have
	$V_{1,T+1,\beta,m}(\omega)=\max\{f_T(\omega),g_T(\omega)\}$ with
	\begin{align}
		\begin{array}{ll}
			f_{T}(\omega)\defeq(1-\epsilon)\omega+(1-\epsilon)\omega\beta V_{1,T,\beta,m}(p_{11})\\[4pt]
			\quad\quad\quad\quad+(1-(1-\epsilon)\omega)\beta V_{1,T,\beta,m}(\mathcal{T}\circ\phi(\omega)),\\[4pt]
			g_{T}(\omega)\defeq m+\beta
			V_{1,T,\beta,m}(\mathcal{T}(\omega)).
		\end{array}   \label{eqn:valueT}
	\end{align}
	From the above, we have
	\begin{gather}
		\begin{array}{ll}
			f_{T}'(\omega)=(1-\epsilon)+(1-\epsilon)\beta V_{1,T,\beta,m}(p_{11})
			-(1-\epsilon)\beta V_{1,T,\beta,m}(\mathcal{T}\circ\phi(\omega))\\[4pt]
			~~~~~~~~~~~~+V_{1,T,\beta,m}'(\mathcal{T}\circ\phi(\omega))\frac{\epsilon\beta(p_{11}-p_{01})}{1-(1-\epsilon)\omega},\\[4pt]
			g_{T}'(\omega)=\beta(p_{11}-p_{01})V_{1,T,\beta,m}'(\mathcal{T}(\omega)),
		\end{array}\label{eqn:diffT}
	\end{gather} where $f_T'(\cdot),\ g_T'(\cdot)$ and $V_{1,T,\beta,m}'(\cdot)$
	denote the right derivatives of the corresponding functions. We have
	used the fact that $\phi(\cdot)$ is monotonically increasing and when
	$p_{11}>p_{01}$, $\Tau(\cdot)$ is also monotonically increasing and
	that
	\begin{eqnarray}
		\phi'(\omega)=\frac{\epsilon}{(1-(1-\epsilon)\omega)^2}.
	\end{eqnarray}
	By the induction hypothesis and~\eqref{eqn:diffT}, if $p_{11}>p_{01}$ then $g_{T}(\omega)$ is monotonically increasing (since $g'_{T}(\omega)\ge0$) and
	\begin{eqnarray}
		\begin{aligned}
			f_{T}'(\omega)&=(1-\epsilon)+(1-\epsilon)\beta V_{1,T,\beta,m}(p_{11})-(1-\epsilon)\beta V_{1,T,\beta,m}(\mathcal{T}(\tfrac{\epsilon\omega}{\epsilon\omega+1-\omega})) \\
			&~~~+(1-(1-\epsilon)\omega)\beta V_{1,T,\beta,m}'(\mathcal{T}(\frac{\epsilon\omega}{\epsilon\omega+1-
				\omega}))(\mathcal{T}\circ\phi)'(\omega) \\
			&\geq (1-\epsilon)+\tfrac{\epsilon\beta(p_{11}-p_{01})}{1-(1-\epsilon)\omega}V_{1,T,\beta,m}'(\mathcal{T}(\tfrac{\epsilon\omega}{\epsilon\omega+1-\omega}))> 0,
		\end{aligned}
	\end{eqnarray}
	where both the first and second inequalities are due to the
	monotonically increasing property of $V_{1,T,\beta,m}(\cdot)$ under
	the assumption that $p_{11}>p_{01}$ by our induction hypothesis and
	\begin{eqnarray}
		p_{01}\le\Tau(\omega)\le p_{11},\ 0\le \phi(\omega)\le1,\quad\forall~\omega\in[0,1].
	\end{eqnarray}
	This proves the monotonically increasing property of $f_T(\omega)$. Thus $V_{1,T+1,\beta,m}(\omega)
	=\max\{f_T(\omega), g_T(\omega)\}$ is also monotonically increasing and the proof by induction is finished.
\end{proof}

Now we show that under a constraint on the discount factor $\beta\in(0,1)$, the value function $V_{1,T,\beta,m}(\omega)$ is a Lipschitz function:
\begin{lemma}\label{lm:LipCont}
	Suppose the discount factor $\beta\in(0,1)$ satisfies
	\begin{eqnarray}\label{eqn:betaCond}
		\beta <\tfrac{1}{(2-\epsilon)\vert p_{11}-p_{01}\vert}.
	\end{eqnarray}
	Then $\forall\ T \geq 1$ and $\forall\ \omega,\omega'\in\lbrack 0,1\rbrack$,
	{\begin{eqnarray}
			&\vert V_{1,T,\beta,m}(\omega)-V_{1,T,\beta,m}(\omega')\vert\leq C\vert\omega-\omega'\vert, \ \text{where}~\\
			&C=\frac{1-\epsilon}{1-(2-\epsilon)\beta\vert p_{11}-p_{01}\vert}.\label{eqn:C}
	\end{eqnarray}}
\end{lemma}
\begin{proof}
	We prove this by induction. Without loss of generality, assume $\omega<\omega'$. For the case of $T=1$,
	\begin{equation*}
		\vert V_{1,1,\beta,m}(\omega)-V_{1,1,\beta,m}(\omega')\vert=
		\begin{cases}
			0,& m\geq (1-\epsilon)\omega'\\
			(1-\epsilon)\omega'-m,&\text{if } (1-\epsilon)\omega\le m<(1-\epsilon)\omega'\\
			(1-\epsilon)\vert\omega-\omega'\vert,&\text{if } m<(1-\epsilon)\omega
		\end{cases}.
	\end{equation*}
	Thus $\vert V_{1,1,\beta,m}(\omega)-V_{1,1,\beta,m}(\omega')\vert\leq(1-\epsilon)\vert\omega-\omega'\vert\leq C\vert\omega-\omega'\vert$, where the second inequality is due to~\eqref{eqn:betaCond}.

	Assume that for $T\geq
	1,~\vert V_{1,T,\beta,m}(\omega)-V_{1,T,\beta,m}(\omega')\vert\leq
	C\vert\omega-\omega'\vert$ holds, \ie neither the left nor the right
	derivative of $V_{1,T,\beta,m}(\cdot)$ can have an absolute value exceeding~$C$. We have the following inequalities:
    	\begin{gather}
		\begin{aligned}
			\vert V_{1,T,\beta,m}(p_{11})-V_{1,T,\beta,m}(\mathcal{T}(\tfrac{\epsilon\omega}{\epsilon\omega+1-\omega}))\vert&\leq C\vert p_{11}-\mathcal{T}(\tfrac{\epsilon\omega}{\epsilon\omega+1-\omega})\vert\\
			&\leq C\vert p_{11}-p_{01}\vert, \\[4pt]
			\frac{\epsilon}{1-(1-\epsilon)\omega}\vert V_{1,T,\beta,m}'(\mathcal{T}(\tfrac{\epsilon\omega}{\epsilon\omega+1-\omega}))\vert&\leq \frac{\epsilon}{1-(1-\epsilon)\cdot1}C= C.
		\end{aligned}\label{eqn:boundsOnV}
	\end{gather}
	\normalsize

To prove
	$\vert V_{1,T+1,\beta,m}(\omega)-V_{1,T+1,\beta,m}(\omega')\vert\leq
	C\vert\omega-\omega'\vert$, recall the definitions of $f_T(\omega)$ and
	$g_T(\omega)$ in~\eqref{eqn:valueT} and their right derivatives
	$f_T'(\omega)$ and $g_T'(\omega)$ in~\eqref{eqn:diffT}. We have
	\begin{gather}
		\begin{aligned}
			|f_{T}'(\omega)-(1-\epsilon)|&\le(1-\epsilon)\beta C\vert p_{11}-p_{01}\vert+C\beta\vert p_{11}-p_{01}\vert=(2-\epsilon)\beta C\vert p_{11}-p_{01}\vert, \\
			|g_{T}'(\omega)| &\leq C\beta\vert p_{11}-p_{01}\vert,
		\end{aligned}
		\label{eqn:diffBound}
	\end{gather}
	\normalsize
	where the inequality in the first (or second) line above is due to the first (or second) line in~\eqref{eqn:boundsOnV}. Thus we have the following lower and upper bounds on $f'_T(\omega)$
	and $g'_T(\omega)$:	
	\begin{gather}
		\begin{aligned}
			(1-\epsilon)-(2-\epsilon)\beta C\vert p_{11}-p_{01}\vert\leq f_{T}'(\omega)&\leq (1-\epsilon)+(2-\epsilon)\beta C\vert p_{11}-p_{01}\vert, \\
			-C\beta\vert p_{11}-p_{01}\vert\leq g_{T}'(\omega) &\leq C\beta\vert p_{11}-p_{01}\vert.
		\end{aligned}
		\label{eqn:diffBound}
	\end{gather}
	\normalsize
From~\eqref{eqn:diffBound} and that~$V_{1,T+1,\beta,m}(\omega)
	=\max\{f_T(\omega), g_T(\omega)\}$, we have
	\[
	\begin{aligned}
		\vert V_{1,T+1,\beta,m}'(\omega)\vert&\leq(1-\epsilon)+(2-\epsilon)\beta C\vert p_{11}-p_{01}\vert\\
		&=(1-\epsilon)+(2-\epsilon)\beta\vert p_{11} -p_{01}\vert\tfrac{1-\epsilon}{1-(2-\epsilon)\beta\vert p_{11}-p_{01}\vert} \\
		&=\tfrac{1-\epsilon}{1-(2-\epsilon)\beta\vert p_{11}-p_{01}\vert}=C,
	\end{aligned}
	\]
where we used the fact that $2-\epsilon\ge1$ in the above inequality.
	Since $V_{1,T+1,\beta,m}(\omega)$ is absolutely continuous, the above implies that
	\[
	\vert V_{1,T+1,\beta,m}(\omega)-V_{1,T+1,\beta,m}(\omega')\vert\leq C\vert\omega-\omega'\vert.
	\]
The proof is thus finished by the induction process.
\end{proof}
Last, we give a lemma establishing the order of
$V'_{1,T,\beta,m}(\,\cdot\,;u=1)$ and
\\$V'_{1,T,\beta,m}(\,\cdot\,;u=0)$ under certain conditions which
further leads to a threshold structure of the optimal single-arm
policy as detailed in Section~\ref{sec:thIdx}.
\begin{lemma}\label{lm:orderDiff}
	Suppose that $p_{11}>p_{01}$ and $\beta\leq 1/[{(3-\epsilon)(p_{11}-p_{01})}]$, we have
	\begin{eqnarray}
		V_{1,T,\beta,m}'(\omega;u=1)\geq V_{1,T,\beta,m}'(\omega;u=0), \label{eqn:diffOrder}
	\end{eqnarray}
	where $V_{1,T,\beta,m}'(\omega;u=i)$ denotes the right derivative of
	$V_{1,T,\beta,m}(\,\cdot\,;u=i)$ at $\omega$ for $i\in\{0,1\}$. The above
	inequality is also true if $p_{01}>p_{11}$ and $\beta\leq 1/[{(5-2\epsilon)(p_{01}-p_{11})}]$.
\end{lemma}
\begin{proof}
	Again, we prove by induction on the time horizon $T$. When~$T=1$, it is clear that $V_{1,1,\beta,m}(\omega;u=1)=(1-\epsilon)\omega$ and $V_{1,1,\beta,m}(\omega;u=0)=m$:
	\begin{eqnarray}
		V_{1,1,\beta,m}'(\omega;u=1)=1-\epsilon>V_{1,1,\beta,m}'(\omega;u=0)=0.
	\end{eqnarray}
	Assume that $V_{1,T,\beta,m}'(\omega;u=1)\ge
	V_{1,T,\beta,m}'(\omega;u=0)$ for~$T\ge1$. From \eqref{eqn:diffBound},
	we have, in case of $p_{01}>p_{11}$ and $\beta\leq \frac{1}{(5-2\epsilon)(p_{01}-p_{11})}$, the inequality $C\beta(p_{01}-p_{11})\leq (1-\epsilon)-(2-\epsilon)\beta C(p_{01}-p_{11}),$
	which shows that $f_{T}'(\omega)\ge g_{T}'(\omega)$. When $p_{11}>p_{01}$, $V_{1,T,\beta,m}(\omega)$ is increasing with $\omega$ therefore has nonnegative right derivatives by Lemma~\ref{lm:posVd}. We can thus obtain tighter bounds on $f_{T}'(\omega)$ and $g_{T}'(\omega)$ by~\eqref{eqn:diffT}:
	\[
	\begin{aligned}
		(1-\epsilon)\leq f_{T}'(\omega)&\leq(1-\epsilon)+(2-\epsilon)\beta C(p_{11}-p_{01}), \\
		0\leq g_{T}'(\omega) &\leq C\beta(p_{11}-p_{01}).
	\end{aligned}
	\]
	If $\beta\leq\frac{1}{(3-\epsilon)(p_{11}-p_{01})}$, we have $C\beta(p_{11}-p_{01})\leq(1-\epsilon)$, which shows that $f_{T}'(\omega)\geq g_{T}'(\omega)$. The proof is thus complete.
\end{proof}

\subsection{Threshold Policy and Indexability}
\label{sec:thIdx}
In this section, we show that the optimal single-arm policy is a threshold policy under the constraints on the discount factor $\beta$ specified in Section~\ref{sec:singleArm} and analyze the conditions for indexability. First, for a finite-horizon single-armed bandit, a threshold policy $\pi$ is defined by a time-dependent real number $\omega_{T,\beta}(m)$ such that
\begin{equation}
	u_{T,m}(\omega)=\begin{cases}
		1, &\text{if } \omega>\omega_{T,\beta}(m);\\
		0, &\text{if } \omega\leq\omega_{T,\beta}(m).
	\end{cases}\label{eqn:thT}
\end{equation}
In the above
$u_{T,m}(\omega)\in\{0,1\}$ is the action taken under
$\pi$ at the current state $\omega$ with
$T$ slots remaining. Intuitively, the larger
$\omega$ is, the larger expected immediate reward to accrue and thus
more attractive to activate the arm. We formalize this intuition under
certain conditions in the following theorem.
\begin{theorem}\label{thm:thresholdT}
	Suppose that $p_{11}>p_{01}$ and $\beta\leq\frac{1}{(3-\epsilon)(p_{11}-p_{01})}$. For any $T\geq 1$, the optimal single-arm policy $\pi^*$ is a threshold policy, \ie there exists $\omega_{T,\beta}^{*}(m)\in\mathbbm{R}$ such that under $\pi^*$, the optimal action is
	\begin{equation*}
		u_{T,m}^{*}(\omega)=\begin{cases}
			1, &\text{if } \omega>\omega_{T,\beta}^{*}(m);\\
			0, &\text{if } \omega\leq\omega_{T,\beta}^{*}(m).
		\end{cases}
	\end{equation*}
	Furthermore, at the threshold $\omega_{T,\beta}^{*}(m)$,
	\begin{eqnarray}
		V_{1,T,\beta,m}(\omega_{T,\beta}^{*}(m);u=0)=V_{1,T,\beta,m}(\omega_{T,\beta}^{*}(m);u=1).
	\end{eqnarray}
	The conclusion is also true for the case of $p_{01}>p_{11}$ and $\beta\leq\frac{1}{(5-2\epsilon)(p_{01}-p_{11})}$.
\end{theorem}
\begin{proof}
	At $T=1$, $V_{1,1,\beta,m}(\omega;u=1)=(1-\epsilon)\omega,V_{1,1,\beta,m}(\omega;u=0)=m$. Thus we can choose $\omega_{1,\beta}^{*}(m)$ as follows:
	\begin{equation*}
		\omega_{1,\beta}^{*}(m)=\begin{cases}
			c, &\text{if } m\geq 1-\epsilon;\\
			\frac{m}{1-\epsilon}, &\text{if } 0\leq m<1-\epsilon; \\
			b, &\text{if } m<0,
		\end{cases}
	\end{equation*}
	where $b<0,c>1$ are arbitrary constants.
	
	For $T\geq 1$, when the condition on $\beta$ is satisfied,
	Lemma~\ref{lm:orderDiff} shows that
	\begin{eqnarray}
		h_T(\omega)&\defeq V_{1,T,\beta,m}(\omega;u=1)-V_{1,T,\beta,m}(\omega;u=0),\label{eqn:vDiff}
	\end{eqnarray}
	\begin{eqnarray}
		h'_T(\omega)&\geq 0,\quad\forall~\omega\in(0,1).\label{eqn:orderAct}
	\end{eqnarray}
	This shows that $h_T(\cdot)$ is monotonically increasing and either
	has no zeros in the interval $[0,1]$ or intersects with it over a
	closed interval (which can be a single point) only. Specially,
	\[
	\begin{aligned}
		V_{1,T,\beta,m}(0;u=1)&=\beta V_{1,T-1,\beta,m}(p_{01}), \\
		V_{1,T,\beta,m}(0;u=0)&=m+\beta V_{1,T-1,\beta,m}(p_{01}), \\
		V_{1,T,\beta,m}(1;u=1)&=(1-\epsilon)+\beta V_{1,T-1,\beta,m}(p_{11}), \\
		V_{1,T,\beta,m}(1;u=0)&=m+\beta V_{1,T-1,\beta,m}(p_{11}).
	\end{aligned}
	\]
	Consider the following three regions of $m$.
	\begin{itemize}
		\item[(i)] $0\leq m<1-\epsilon$. In this case, $V_{1,T,\beta,m}(0;u=1)\leq V_{1,T,\beta,m}(0;u=0)$ and $V_{1,T,\beta,m}(1;u=1)>V_{1,T,\beta,m}(1;u=0)$. Therefore $h_T(\cdot)$ intersects over (at least) one point in $[0,1]$. This point can thus be chosen as $\omega_{T,\beta}^{*}(m)$.
		\item[(ii)]$m<0$. In this case, $V_{1,T,\beta,m}(0;u=1)> V_{1,T,\beta,m}(0;u=0)$ and $V_{1,T,\beta,m}(1;u=1)>V_{1,T,\beta,m}(1;u=0)$. So $h_T(\cdot)$ is strictly positive over $[0,1]$ and we can choose $\omega_{T,\beta}^{*}(m)=b$ with any $b<0$.
		\item[(iii)] $m\geq(1-\epsilon)$. In this case, always choosing the
		passive action is clearly optimal as the expected immediate reward
		is uniformly upper-bounded by $m$ over the whole belief state
		space. We can thus choose $\omega_{T,\beta}^{*}(m)=c$ with any
		$c>1$.
	\end{itemize}
	In conclusion, when the conditions in the theorem are satisfied, the optimal finite-horizon single-arm policy is a threshold policy for any horizon length $T\ge1$.
\end{proof}
In the next theorem, we show that the optimal single-arm policy over
the {\em infinite} horizon is also a threshold policy under the same
conditions.
\begin{theorem}\label{thm:threshold}
	Fix the subsidy $m$. The finite-horizon value functions
	$V_{1,T,\beta,m}(\cdot)$,
	$V_{1,T,\beta,m}(\,\cdot\,;u=1)\text{~and~} V_{1,T,\beta,m}(\,\cdot\,;u=0)$
	uniformly converge to the infinite-horizon value functions
	$V_{\beta,m}(\cdot)$,
	$V_{\beta,m}(\,\cdot\,;u=1)\text{~and~} V_{\beta,m}(\,\cdot\,;u=0)$ which
	are consequently obedient to the same properties established in
	Lemmas \ref{lm:posVd} and~\ref{lm:LipCont} and
	Theorem~\ref{thm:thresholdT}.
\end{theorem}
\begin{proof}
The uniform convergence is obvious since $\beta<1$ and the rest can be easily proved by contradiction following the uniform convergence.
\end{proof}
Thus far we have established the threshold structure of the optimal
single-arm policy with subsidy based on the analysis of
$V_{\beta,m}(\omega)$ as a function of the belief state
$\omega$ with
$m$ fixed. To study the indexability condition, we now analyze the
properties of $V_{\beta,m}(\omega)$ as a function of the subsidy
$m$ with the starting belief
$\omega$ fixed. From Definition~\ref{def:indexability} and the
threshold structure of the optimal policy, the indexability of our
model is reduced to requiring that the threshold
$\omega^*_{\beta}(m)$ is monotonically increasing with
$m$ (if the threshold is a closed interval then the right end is
selected). Note that for the infinite-horizon problem, the threshold
$\omega^*_{\beta}(m)$ is independent of time. Furthermore,
$V_{\beta,m}(\omega)$ is also convex in
$m$ as for any $m_1,m_2\in\mathbbm{R}$ and
$\theta\in(0,1)$ the optimal policy $\pi_\beta^*(\theta
m_1+(1-\theta)m_2)$ achieving $V_{\beta,\theta
	m_1+(1-\theta)m_2}(\omega)$ applied respectively on the problem
with subsides $m_1$ and
$m_2$ cannot outperform those achieving
$V_{\beta,m_1}(\omega)$ and
$V_{\beta,m_2}(\omega)$. Specifically, let
$r_a$ be the expected total discounted reward from the active action
and $r_p(m)$ that from the passive action under $\pi_\beta^*(\theta
m_1+(1-\theta)m_2)$ applied to the problem with subsidy $m$, then
\begin{eqnarray}\nn
	\theta V_{\beta,m_1}(\omega) + (1-\theta)V_{\beta,m_2}(\omega)&\ge& r_a+\theta r_p(m_1)+(1-\theta)r_p(m_2)\\ \nn
	&=& r_a+r_p(\theta m_1+(1-\theta)m_2)\\\nn
	&=&V_{\beta,\theta m_1+(1-\theta)m_2}(\omega).
\end{eqnarray}
Since $V_{\beta,m}(\omega)$ is convex in $m$, its left and right derivatives with $m$ exist at every point $m_0\in\mathbbm{R}$. Furthermore, consider two policies $\pi_\beta^*(m_1)$ and $\pi_\beta^*(m_2)$ achieving $V_{\beta,m_1}(\omega)$ and $V_{\beta,m_2}(\omega)$ for any $m_1,m_2\in\mathbbm{R}$, respectively. With a similar interchange argument of $\pi_\beta^*(m_1)$ and $\pi_\beta^*(m_2)$ as above, we have $\vert V_{\beta,m_1}(\omega)-V_{\beta,m_2}(\omega)\vert\le \frac{1}{1-\beta}\vert m_1-m_2\vert$ and $V_{\beta,m}(\omega)$ is Lipschitz continuous in $m$. By the Rademacher theorem (see \citealp{H2005}), $V_{\beta,m}(\omega)$ is differentiable almost everywhere in $m$. For a small increase of $m$, the rate at which $V_{\beta,m}(\omega)$ increases is at least the expected total discounted passive time under any optimal policy for the problem with subsidy $m$ starting from the belief state $\omega$. In the following theorem, we formalize this relation between the value function and the passive time as well as a sufficient condition for the indexability of our model.
\begin{theorem}\label{thm:indexability}
	Let~$\Pi^*_{\beta}(m)$ denote the set of all optimal single-arm policies achieving~$V_{\beta,m}(\omega)$ with initial belief state~$\omega$. Define the passive time
	\begin{eqnarray}
		D_{\beta,m}(\omega)\defeq
		\max_{\pi^*_{\beta}(m)\in\Pi^*_{\beta}(m)}\mathbbm{E}_{\pi^*_{\beta}(m)}\left[\sum_{t=1}^{\infty}\beta^{t-1}\mathbbm{1}_{(u(t)=0)}\Big\vert
		\omega(1)=\omega\right].\label{def:passiveTime}
	\end{eqnarray}
	The right derivative of the value function~$V_{\beta,m}(\omega)$ with~$m$, denoted by~$\frac{dV_{\beta,m}(\omega)}{(dm)^+}$, exists at every value of~$m$ and
	\begin{eqnarray}
		\left.\frac{dV_{\beta,m}(\omega)}{(dm)^+}\right\vert_{m=m_0}=D_{\beta,m_0}(\omega).\label{eq:rightDiff}
	\end{eqnarray}
	Furthermore, the single-armed bandit is indexable if at least one of the following condition is satisfied:
	\begin{itemize}
		\item[i.] for any $m_0\in[0,1-\epsilon)$ the optimal policy is a threshold policy with threshold $\omega^*_{\beta}(m_0)\in[0,1)$ (if the threshold is a closed interval then the right end is selected) and
		\begin{eqnarray}
			\left.\frac{dV_{\beta,m}(\omega^*_{\beta}(m_0);u=0)}{(dm)^+}\right\vert_{m=m_0} > \left.\frac{dV_{\beta,m}(\omega^*_{\beta}(m_0);u=1)}{(dm)^+}\right\vert_{m=m_0}.\label{eq:diffIdx}
		\end{eqnarray}
		\item[ii.] for any $m_0\in\mathbbm{R}$ and $\omega\in P(m_0)$, we have
		\begin{eqnarray}
			\left.\frac{dV_{\beta,m}(\omega;u=0)}{(dm)^+}\right\vert_{m=m_0} \ge \left.\frac{dV_{\beta,m}(\omega;u=1)}{(dm)^+}\right\vert_{m=m_0}.\label{eq:diffIdx1}
		\end{eqnarray}
	\end{itemize}
\end{theorem}
\begin{proof}
	The proof of~\eqref{eq:rightDiff} follows directly from the argument in Theorem 1 in \citet{L2021} and is omitted here. To prove the sufficiency of~\eqref{eq:diffIdx}, we note that if it is true then there exists a $\Delta m>0$ such that $\forall~m\in(m_0,m_0+\Delta m)$,
	\begin{eqnarray}
		&~~~V_{\beta,m}(\omega^*_{\beta}(m_0);u=0)- V_{\beta,m_0}(\omega^*_{\beta}(m_0);u=0)  \nonumber \\
		&>V_{\beta,m}(\omega^*_{\beta}(m_0);u=1)-V_{\beta,m_0}(\omega^*_{\beta}(m_0);u=1).\nonumber
	\end{eqnarray}
	Since $V_{\beta,m_0}(\omega^*_{\beta}(m_0);u=0)=V_{\beta,m_0}(\omega^*_{\beta}(m_0);u=1)$, we have $V_{\beta,m}(\omega^*_{\beta}(m_0);u=0)>V_{\beta,m}(\omega^*_{\beta}(m_0);u=1)$ which implies that the threshold $\omega^*_{\beta}(m_0)$ remains in the passive set as $m$ continuously increases so $P(m)$ is monotonically increasing with $m$. This conclusion is clearly true for the trivial case of $m<0$ or $m\ge1-\epsilon$. The sufficiency of~\eqref{eq:diffIdx1} is obvious because then it is impossible for any $\omega\in P(m)$ to escape from $P(m)$ as $m$ increases due to the nondecreasing property of $V_{\beta,m}(\omega;u=0)-V_{\beta,m}(\omega;u=1)$ enforced by~\eqref{eq:diffIdx1}.
\end{proof}
Theorem~\ref{thm:indexability} essentially provides a way for checking the indexability condition in terms of the passive times. For example, equation~\eqref{eq:diffIdx} is equivalent to for any $m\in[0,1-\epsilon)$,
{\small\begin{eqnarray}
		&\beta\Bigl[(1-\epsilon)\omega_{\beta}^{*}(m)D_{\beta,m}(p_{11}) +(1-(1-\epsilon)\omega_{\beta}^{*}(m))D_{\beta,m}\Bigl(\mathcal{T}\bigl(\tfrac{\epsilon\omega_{\beta}^{*}(m)}{\epsilon\omega_{\beta}^{*}(m)+1-\omega_{\beta}^{*}(m)}\bigr)\Bigr)\Bigr]\nn\\
		&<1+\beta D_{\beta,m}(\mathcal{T}(\omega_{\beta}^{*}(m))).\label{eqn:orderPT}
 \end{eqnarray}}
The above strict inequality clearly holds if $\beta<0.5$ since $D_{\beta,m}(\cdot)\in[0,\frac{1}{1-\beta}]$ for any $m\in\mathbbm{R}$. When $\beta=0.5$, we prove by contradiction that the strict inequality~\eqref{eqn:orderPT} must hold under the threshold structure of the optimal policy. If $\omega_{\beta}^{*}(m)=0$ then~\eqref{eqn:orderPT} is clearly true. Assume that the left and right sides of~\eqref{eqn:orderPT} are equal and $\omega_{\beta}^{*}(m)\neq0$. In this case, we have
\begin{eqnarray}
	&&D_{\beta,m}(p_{11}) = \tfrac{1}{1-\beta},\label{eqn:d1}\\
	&&D_{\beta,m}(\mathcal{T}(\omega_{\beta}^{*}(m)))=0.\label{eqn:2}
\end{eqnarray}
Equation~\eqref{eqn:2} implies that starting from
$\mathcal{T}(\omega_{\beta}^{*}(m))$, always activating the arm is
strictly optimal. This means that the threshold
$\omega_{\beta}^{*}(m)$ is strictly below $p_{11}$ and we have a
contradiction to~\eqref{eqn:d1}. Another easier way to see that the
bandit is indexable if $\beta\le0.5$ is that \eqref{eq:diffIdx1} would
be satisfied where no strict inequality is required. However,
condition~\eqref{eqn:orderPT} provides a convenient way for
approximately computing the passive times as well as the value
functions which leads to an efficient algorithm for evaluating the
indexability and solving for the Whittle index function for any
$\beta\in(0,1)$, as detailed in the next section.
\begin{corollary}
	The restless bandit is indexable if~$\beta\le 0.5$.
\end{corollary}

\section{The Whittle Index Policy}\label{sec:WhittleAlgor}

In this section, we design an efficient algorithm by approximating the Whittle index.

\subsection{The Approximated Whittle Index}
The threshold structure of the optimal single-arm policy under certain conditions yields the following iterative nature of the dynamic equations for both $D_{\beta,m}(\omega)$ and $V_{\beta,m}(\omega)$. Define \emph{the first crossing time}
\begin{equation}
	L(\omega,\omega')=\min_{0\le k<\infty}\{k:\mathcal{T}^{k}(\omega)>\omega'\}.
\end{equation}
In the above $\Tau^0(\omega)\defeq\omega$ and we set $L(\omega,\omega')=+\infty$ if $\Tau^k(\omega)\le\omega$ for all $k\ge0$. Clearly $L(\omega,\omega')$ is the minimum time slots required for a belief state $\omega$ to stay in the passive set $P(m)$ before the arm is activated given a threshold $\omega'\in[0,1)$. Consider the nontrivial case where $p_{01},p_{11}\in(0,1)$ and $p_{01}\neq p_{11}$ such that the Markov chain of the internal arm states is aperiodic and irreducible and that the belief update is action-dependent. From~\eqref{eqn:kTau}, if $p_{11}> p_{01}$ then
\begin{equation}
	L(\omega,\omega')=
	\begin{cases}
		0, & \omega>\omega' \\
		\bigg\lfloor \log_{p_{11}-p_{01}}^{\frac{p_{01}-\omega'(1-p_{11}+p_{01})}{p_{01}-\omega(1-p_{11}+p_{01})}} \bigg\rfloor +1, & \omega \leq \omega' <\omega_{o} \\
		\infty, & \omega \leq \omega',\omega' \geq \omega_{o}
	\end{cases};\label{eqn:posL}
\end{equation}
or if $p_{11}<p_{01}$ then
\begin{equation}
	L(\omega,\omega')=
	\begin{cases}
		0, & \omega>\omega' \\
		1, & \omega \leq \omega',\mathcal{T}(\omega)>\omega' \\
		\infty, & \omega \leq \omega',\mathcal{T}(\omega)\leq\omega'
	\end{cases}.\label{eqn:negL}
\end{equation}
To illustrate~\eqref{eqn:posL} and~\eqref{eqn:negL}, note that a belief state will converge to the stationary belief value given by~\eqref{eq:stationary} monotonically for $p_{11}>p_{01}$ or in an oscillating manner for $p_{11}<p_{01}$ under passive actions (see Fig.~3 and Fig.~4 in \cite{LZ2010}). Suppose that the following conditions are satisfied such that the optimal single-arm policy is a threshold policy and the indexability holds:
\begin{equation}
	\beta\le
	\begin{cases}
		\min\left\{\frac{1}{(3-\epsilon)(p_{11}-p_{01})},0.5\right\}, & \text{if}~p_{11}> p_{01} \\[4pt]
		\min\left\{\frac{1}{(5-2\epsilon)(p_{01}-p_{11})},0.5\right\}, & \text{if}~p_{11}< p_{01}
	\end{cases}.\label{eqn:betaThIdx}
\end{equation}
To solve for the Whittle index function $W(\omega)$, given the current
arm state $\omega$, we aim to find out the minimum subsidy $m$ that
makes it as a threshold:
\begin{eqnarray}
	&V_{\beta,m}(\omega)=V_{\beta,m}(\omega;u=1)=V_{\beta,m}(\omega;u=0),\label{eqn:equalVatTh}~~~~~\\
	&V_{\beta,m}(\omega;u=1)=(1-\epsilon)\omega+\beta[(1-\epsilon)\omega V_{\beta,m}(p_{11})
	\notag\\&~~~~~~~~~~~~~~~~~~~~~~~~~~+(1-(1-\epsilon)\omega)V_{\beta,m}(\mathcal{T}\circ\phi(\omega))], \label{eqn:v1Expand}\\[4pt]
	&V_{\beta,m}(\omega;u=0)=m+\beta V_{\beta,m}(\mathcal{T}(\omega)).\label{eqn:v0Expand}~~~~~~~~~~~~~~~~
\end{eqnarray}
Given a threshold $\omega^*_\beta(m)\in[0,1)$ and
any~$\omega\in[0,1]$, the value function $V_{\beta,m}(\omega)$ can be
expanded by the first crossing time as
\begin{gather}\label{eqn:vExpand}
	\begin{aligned}
		V_{\beta,m}(\omega)&=\frac{1-\beta^{L(\omega,\omega^*_\beta(m))}}{1-\beta}m+\beta^{L(\omega,\omega^*_\beta(m))}V_{\beta,m}(\mathcal{T}^{L(\omega,\omega^*_\beta(m))}(\omega);u=1)\\
		&=\frac{1-\beta^{L(\omega,\omega^*_\beta(m))}}{1-\beta}m+\beta^{L(\omega,\omega^*_\beta(m))}\bigg\{(1-\epsilon)\mathcal{T}^{L(\omega,\omega^*_\beta(m))}(\omega)  \\
		&\quad+\beta\bigg[(1-\epsilon)\mathcal{T}^{L(\omega,\omega^*_\beta(m))}(\omega)V_{\beta,m}(p_{11})+(1-(1-\epsilon)\mathcal{T}^{L(\omega,\omega^*_\beta(m))}(\omega)) \\
		&\quad\quad\quad\quad\times V_{\beta,m}\bigg(\mathcal{T}\bigg(\tfrac{\epsilon \mathcal{T}^{L(\omega,\omega^*_\beta(m))}(\omega)}{\epsilon \mathcal{T}^{L(\omega,\omega^*_\beta(m))}(\omega)+1-\mathcal{T}^{L(\omega,\omega^*_\beta(m))}(\omega)}\bigg)\bigg)\bigg]\bigg\}.
	\end{aligned}
\end{gather}
\normalsize There is no doubt that the last item of the above equation
has caused us trouble in solving for $V_{\beta,m}(\omega)$. However,
if we let
\begin{gather}
	\begin{aligned}
		f(\omega, \omega_{\beta}^{*}(m))&=\mathcal{T}\bigg(\tfrac{\epsilon \mathcal{T}^{L(\omega,\omega_{\beta}^{*}(m))}(\omega)}{\epsilon \mathcal{T}^{L(\omega,\omega_{\beta}^{*}(m))}(\omega)+1-\mathcal{T}^{L(\omega,\omega_{\beta}^{*}(m))}(\omega)}\bigg)  \\
		&=\frac{p_{11}\epsilon \mathcal{T}^{L(\omega,\omega_{\beta}^{*}(m))}(\omega)+p_{01}(1-\mathcal{T}^{L(\omega,\omega_{\beta}^{*}(m))}(\omega))}{\epsilon \mathcal{T}^{L(\omega,\omega_{\beta}^{*}(m))}(\omega)+1-\mathcal{T}^{L(\omega,\omega_{\beta}^{*}(m))}(\omega)}
	\end{aligned}
\end{gather}\normalsize
and construct iteratively the sequence $\{k_{n}\}$ as
$k_{n+1}=f(k_{n},\omega_{\beta}^{*}(m))$ with $k_{0}=\omega$. We then
get the following sequence of equations:
\[
{\begin{aligned}
		V_{\beta,m}(k_{0})&=\frac{1-\beta^{L(k_{0},\omega_{\beta}^{*}(m))}}{1-\beta}m+\beta^{L(k_{0},\omega_{\beta}^{*}(m))}\big\{(1-\epsilon)\mathcal{T}^{L(k_{0},\omega_{\beta}^{*}(m))}(k_{0})+\beta\big[(1-\epsilon)\\
		&\quad\times\mathcal{T}^{L(k_{0},\omega_{\beta}^{*}(m))}(k_{0})V_{\beta,m}(p_{11})+(1-(1-\epsilon)\mathcal{T}^{L(k_{0},\omega_{\beta}^{*}(m))}(k_{0}))V_{\beta,m}(k_{1})\big]\big\}
		\\
		V_{\beta,m}(k_{1})&=\frac{1-\beta^{L(k_{1},\omega_{\beta}^{*}(m))}}{1-\beta}m+\beta^{L(k_{1},\omega_{\beta}^{*}(m))}\big\{(1-\epsilon)\mathcal{T}^{L(k_{1},\omega_{\beta}^{*}(m))}(k_{1})+\beta\big[(1-\epsilon)\\
		&\quad\times\mathcal{T}^{L(k_{1},\omega_{\beta}^{*}(m))}(k_{1})V_{\beta,m}(p_{11})+(1-(1-\epsilon)\mathcal{T}^{L(k_{1},\omega_{\beta}^{*}(m))}(k_{1}))V_{\beta,m}(k_{2})\big]\big\}
		\\
		\cdots
		\\
		V_{\beta,m}(k_{n})&=\frac{1-\beta^{L(k_{n},\omega_{\beta}^{*}(m))}}{1-\beta}m+\beta^{L(k_{n},\omega_{\beta}^{*}(m))}\big\{(1-\epsilon)\mathcal{T}^{L(k_{n},\omega_{\beta}^{*}(m))}(k_{n})\\
		&\quad+\beta\big[(1-\epsilon)\cdot\mathcal{T}^{L(k_{n},\omega_{\beta}^{*}(m))}(k_{n})V_{\beta,m}(p_{11})
		+(1-(1-\epsilon)\\
		&\quad\quad\quad\times\mathcal{T}^{L(k_{n},\omega_{\beta}^{*}(m))}(k_{n}))V_{\beta,m}(k_{n+1})\big]\big\}
		\\
		\cdots
\end{aligned}}
.
\]
For sufficiently large $n$, we can get an estimation of
$V_{\beta,m}(\omega)=V_{\beta,m}(k_{0})$ with an arbitrarily small
error by setting $V_{\beta,m}(k_{n+1})=0$ whose error is discounted by
$\beta$ in computing $V_{\beta,m}(k_{n})$ thus causing a geometrically
decreasing error propagation in the backward computation process for
$V_{\beta,m}(k_0)$. Note that we first compute $V_{\beta,m}(p_{11})$
in the same way by setting $k_0=p_{11}$ in the above equation
set. Therefore we can have an estimation of $V_{\beta,m}(\omega)$ with
arbitrarily high precision for any $\omega\in[0,1]$. Interestingly,
extensive numerical results found that $\{k_{n}\}$ quickly converges
to a limit belief state~$k^*$ (independent of $k_0$) (see Fig.~\ref{fig:convergeP} for an example). Specifically, after~$4$
iterations, the difference $\vert k_4-k\vert$ becomes too small to affect the performance of our algorithm as discussed in Sec.~\ref{sec:experiments}. So we can set
$V_{\beta,m}(k_{5})=V_{\beta,m}(k_{4})$ and efficiently solve the {\em
	finite} linear equation set (up to $V_{\beta,m}(k_{4})$). In
general, the $n$-iteration Whittle index is based on the solution of
the following equations:
\begin{eqnarray}\nn
	V_{\beta,m}(p_{11})&&=\frac{1-\beta^{L(p_{11},\omega_{\beta}^{*}(m))}}{1-\beta}m+\beta^{L(p_{11},\omega_{\beta}^{*}(m))}\big\{(1-\epsilon)\mathcal{T}^{L(p_{11},\omega_{\beta}^{*}(m))}(p_{11})\\\nn
	&&\quad+\beta\big[(1-\epsilon)\mathcal{T}^{L(p_{11},\omega_{\beta}^{*}(m))}(p_{11})V_{\beta,m}(p_{11})
	+(1-(1-\epsilon)\\\nn
	&&\quad\quad\quad\times\mathcal{T}^{L(p_{11},\omega_{\beta}^{*}(m))}(p_{11}))V_{\beta,m}(k_{1})\big]\big\}
	\\\nn
	V_{\beta,m}(k_{1})&&=\frac{1-\beta^{L(k_{1},\omega_{\beta}^{*}(m))}}{1-\beta}m+\beta^{L(k_{1},\omega_{\beta}^{*}(m))}\big\{(1-\epsilon)\mathcal{T}^{L(k_{1},\omega_{\beta}^{*}(m))}(k_{1})+\beta\big[(1-\epsilon)\\\nn
	&&\quad\times\mathcal{T}^{L(k_{1},\omega_{\beta}^{*}(m))}(k_{1})V_{\beta,m}(p_{11})+(1-(1-\epsilon)\mathcal{T}^{L(k_{1},\omega_{\beta}^{*}(m))}(k_{1}))V_{\beta,m}(k_{2})\big]\big\}
	\\\nn
	&& \vdots
	\\\nn
	V_{\beta,m}(k_{n})&&=\frac{1-\beta^{L(k_{n},\omega_{\beta}^{*}(m))}}{1-\beta}m+\beta^{L(k_{n},\omega_{\beta}^{*}(m))}\big\{(1-\epsilon)\mathcal{T}^{L(k_{n},\omega_{\beta}^{*}(m))}(k_{n})\\\nn
	&&\quad+\beta\big[(1-\epsilon)\mathcal{T}^{L(k_{n},\omega_{\beta}^{*}(m))}(k_{n})V_{\beta,m}(p_{11})
	+(1-(1-\epsilon)\\\nn
	&&\quad\quad\quad\times\mathcal{T}^{L(k_{n},\omega_{\beta}^{*}(m))}(k_{n}))V_{\beta,m}(k_{n})\big]\big\}.
\end{eqnarray}\normalsize
After the value functions are (approximately) solved, we can plot the active value function~$V_{\beta,m}(\omega;u=1)$ and the passive one~$V_{\beta,m}(\omega;u=0)$ to see that they intersect at one single point (the threshold) verifying the optimality of the threshold policy proven by Theorem~\ref{thm:threshold} (See Fig.~\ref{fig:thresholdP} for an example). According to Theorem~\ref{thm:indexability}, the passive time
$D_{\beta,m}(\omega)$ can also be approximately solved based on the
following equations:
\[
\begin{aligned} D_{\beta,m}(p_{11})&=\frac{1-\beta^{L(p_{11},\omega_{\beta}^{*}(m))}}{1-\beta}+\beta^{L(p_{11},\omega_{\beta}^{*}(m))+1}\big\{(1-\epsilon)\mathcal{T}^{L(p_{11},\omega_{\beta}^{*}(m))}(p_{11})\\
	&\quad\times D_{\beta,m}(p_{11})+(1-(1-\epsilon)\mathcal{T}^{L(p_{11},\omega_{\beta}^{*}(m))}(p_{11}))D_{\beta,m}(k_{1})\big\}
	\\
	D_{\beta,m}(k_{1})&=\frac{1-\beta^{L(k_{1},\omega_{\beta}^{*}(m))}}{1-\beta}+\beta^{L(k_{1},\omega_{\beta}^{*}(m))+1}\big\{(1-\epsilon)\mathcal{T}^{L(k_{1},\omega_{\beta}^{*}(m))}(k_{1})\\
	&\quad\times D_{\beta,m}(p_{11})+(1-(1-\epsilon)\mathcal{T}^{L(k_{1},\omega_{\beta}^{*}(m))}(k_{1}))D_{\beta,m}(k_{2})\big\}
	\\
	\vdots
	\\
	D_{\beta,m}(k_{n})&=\frac{1-\beta^{L(k_{n},\omega_{\beta}^{*}(m))}}{1-\beta}+\beta^{L(k_{n},\omega_{\beta}^{*}(m))+1}\big\{(1-\epsilon)\mathcal{T}^{L(k_{n},\omega_{\beta}^{*}(m))}(k_{n})\\	
	&\quad\times D_{\beta,m}(p_{11})+(1-(1-\epsilon)\mathcal{T}^{L(k_{n},\omega_{\beta}^{*}(m))}(k_{n}))D_{\beta,m}(k_{n})\big\}
\end{aligned}.
\]
Substituting $\omega$ for $\omega^*_\beta(m)$ in the above $n+1$ linear equations with $n+1$ unknowns (first solving for $\omega=p_{11}$), we can obtain $V_{\beta,m}(\omega')$ and $D_{\beta,m}(\omega')$ for any $\omega'\in[0,1]$ according to the linear equation sets. The indexability condition~\eqref{eq:diffIdx} in Theorem~\ref{thm:indexability} can be checked {\em online}: for the original multi-armed bandit problem and for each arm at state $\omega(t)$ at time $t$, we compute its approximated Whittle index $W(\omega(t))$ by solving a set of linear equations, which has a polynomial complexity of the iteration number $n$, independent of the decision time $t$. At time $t$, for each arm, if $W(\cdot)$ is found to be nondecreasing with the arm states $(\omega(1),\omega(2),\ldots,\omega(t))$ appeared so far starting from the initial belief vector $\boldsymbol{\omega}(1)$ defined in~\eqref{eqn:beliefVector}, then the indexability has not been violated. Interestingly, extensive numerical studies have shown that the indexability is always satisfied as illustrated in Figs.~\ref{fig:whittleN} and~\ref{fig:whittleP} in Sec.~\ref{sec:experiments}.

\begin{figure}
	\begin{minipage}[h]{0.5\linewidth}
		\centering
		\includegraphics[height=5.3cm,width=5.3cm]{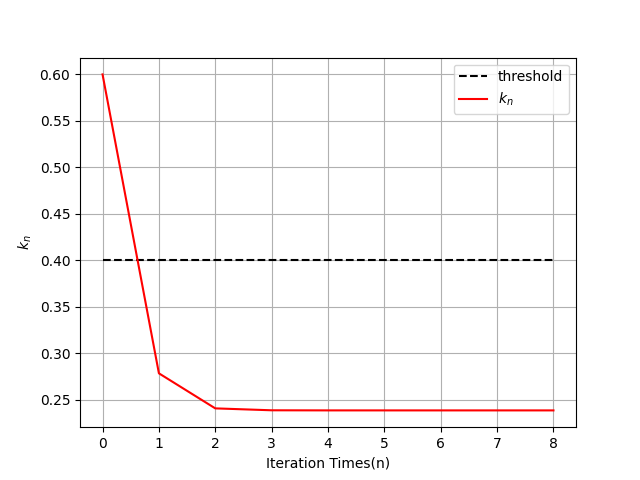}
		\caption{The Convergence of~$k_n$}
		\label{fig:convergeP}
	\end{minipage}%
	\begin{minipage}[h]{0.5\linewidth}
		\centering
		\includegraphics[height=5.3cm,width=5.3cm]{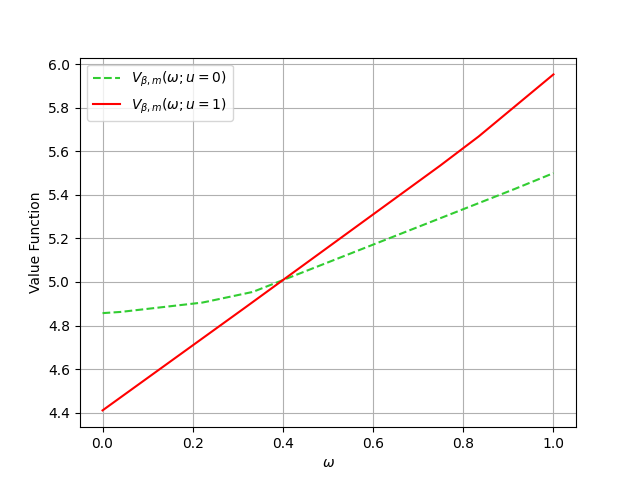}
		\caption{The Optimality of Threshold Policy}
        \label{fig:thresholdP}
	\end{minipage}
\end{figure}

For large $\beta\in(0,1)$ where the threshold structure of the optimal policy or the indexability may not hold (\ie condition~\eqref{eqn:betaThIdx} is not satisfied), we can still use the above process to solve for the subsidy $m$ that makes \eqref{eqn:equalVatTh} true if it exists. Note that after computing the value functions appeared in~\eqref{eqn:equalVatTh} in terms of $m$, both $V_{\beta,m}(\omega;u=1)$ and $V_{\beta,m}(\omega;u=0)$ are linear (affine) in $m$ and their equality gives a \emph{unique} solution of $m$ if their linear coefficients are not equal. This $m$, if exists, can thus be used as the approximated Whittle index $W(\omega)$ \emph{without} requiring indexability or threshold-based optimal policy. If it does not exist, we can simply set $W(\omega)=\omega B$. The existence of such an $m$ is defined as \emph{the relaxed indexability} in~\citet{L2021}. Note that extensive numerical studies have shown that the relaxed indexability of our model with imperfect state observations is always satisfied as well. Before summarizing our general algorithm for all $\beta\in(0,1)$ in Section~\ref{sec:alg}, we solve for the approximated Whittle index function in closed-form for the simplest case of $0$-iteration, which is referred to as \emph{the imperfect Whittle index}. Note that if $\epsilon\rightarrow 0$ then $\mathcal{T}(\frac{\epsilon\omega}{\epsilon\omega+1-\omega})\rightarrow p_{01}$. Thus when $\epsilon$ is sufficiently small, we can approximate $V_{\beta,m}(\mathcal{T}(\frac{\epsilon\omega}{\epsilon\omega+1-\omega}))$ by $V_{\beta,m}(p_{01})$. Under this approximation, we have, for any $\omega\in[0,1]$,
\[
\begin{aligned}
	&V_{\beta,m}(\omega;u=1)=(1-\epsilon)\omega+\beta\big[(1-\epsilon)\omega V_{\beta,m}(p_{11})+(1-(1-\epsilon)\omega)V_{\beta,m}(p_{01})\big] \\
	&V_{\beta,m}(\omega;u=0)=m+\beta V_{\beta,m}(\mathcal{T}(\omega)) \\
	&V_{\beta,m}(\omega)=\frac{1-\beta^{L(\omega,\omega_{\beta}^{*}(m))}}{1-\beta}m+\beta^{L(\omega,\omega_{\beta}^{*}(m))}\big\{(1-\epsilon)\mathcal{T}^{L(\omega,\omega_{\beta}^{*}(m))}(\omega)\\
	&\quad\quad\quad\quad~~+\beta\big[(1-\epsilon)\mathcal{T}^{L(\omega,\omega_{\beta}^{*}(m))}(\omega)V_{\beta,m}(p_{11})\\
	&\quad\quad\quad\quad~~+(1-(1-\epsilon)\mathcal{T}^{L(\omega,\omega_{\beta}^{*}(m))}(\omega))V_{\beta,m}(p_{01})\big]\big\}\\
	%	&\quad\quad\quad\quad\quad\quad~~+(1-(1-\epsilon)\mathcal{T}^{L(\omega,\omega_{\beta}^{*}(m))}(\omega)) V_{\beta,m}(p_{01})\big]\big\}
\end{aligned}.
\]
By using the above three equations, we can directly solve for $V_{\beta,m}(p_{01})$ and $V_{\beta,m}(p_{11})$ in closed-form.\\
When $p_{11}>p_{01}$, $V_{\beta,m}(p_{01})=$
\begin{equation*}
	\begin{cases}
		\frac{(1-\epsilon)p_{01}}{(1-\beta)(1-\beta(1-\epsilon)p_{11}+\beta(1-\epsilon)p_{01})},  \text{if } \omega_{\beta}^{*}(m)<p_{01}\\[10pt]
		\frac{(1-\beta(1-\epsilon)p_{11})(1-\beta^{L(p_{01},\omega_{\beta}^{*}(m))})m+(1-\epsilon)(1-\beta)\beta^{L(p_{01},\omega_{\beta}^{*}(m))}\mathcal{T}^{L(p_{01},\omega_{\beta}^{*}(m))}(p_{01})}{(1-\beta(1-\epsilon)p_{11})(1-\beta)(1-\beta^{L(p_{01},\omega_{\beta}^{*}(m))+1})+(1-\epsilon)(1-\beta)^{2}\beta^{L(p_{01},\omega_{\beta}^{*}(m))+1}\mathcal{T}^{L(p_{01},\omega_{\beta}^{*}(m))}(p_{01})}
		, \\~~~~~~~~~~~~~\text{if } p_{01}\leq\omega_{\beta}^{*}(m)<\omega_{0}\\
		\frac{m}{1-\beta},  \text{if } \omega_{\beta}^{*}(m)\geq\omega_{0}
	\end{cases},
\end{equation*}
\begin{equation*}
	V_{\beta,m}(p_{11})=\begin{cases}
		\frac{(1-\epsilon)p_{11}+\beta(1-(1-\epsilon)p_{11})V_{\beta,m}(p_{01})}{1-\beta(1-\epsilon)p_{11}}, &\text{if }\omega_{\beta}^{*}(m)<p_{11}\\
		\frac{m}{1-\beta}, &\text{if }\omega_{\beta}^{*}(m)\geq p_{11}
	\end{cases}.
\end{equation*}
The approximate Whittle index is given by
{\begin{equation*}
		W(\omega)=\begin{cases}
			\frac{\omega(1-\epsilon)(1-\beta p_{11}+\beta p_{01})}{1-\beta(1-\epsilon)p_{11}+\beta(1-\epsilon)p_{01}}, &\text{if }\omega\leq p_{01}\\
			\frac{(1-\epsilon)(\omega-\beta\mathcal{T}(\omega))+C_{2}(1-\beta)\beta\lbrack(1-\beta(1-\epsilon)p_{11})-(1-\epsilon)(\omega-\beta\mathcal{T}(\omega))\rbrack}{1-\beta(1-\epsilon)p_{11}-C_{1}\beta\lbrack(1-\beta(1-\epsilon)p_{11})-(1-\epsilon)(\omega-\beta\mathcal{T}(\omega))\rbrack}, &\text{if }p_{01}<\omega \leq \omega_{0}\\
			\frac{(1-\epsilon)\omega}{1-\beta(1-\epsilon)p_{11}+\beta(1-\epsilon)\omega}, &\text{if }\omega_{0} < \omega \leq p_{11}\\
			(1-\epsilon)\omega, &\text{if } \omega > p_{11}
		\end{cases},
\end{equation*}}
where
\[
\begin{aligned}
	C_{1}=\frac{(1-\beta(1-\epsilon)p_{11})(1-\beta^{L(p_{01},\omega)})}{(1-\beta(1-\epsilon)p_{11})(1-\beta^{L(p_{01},\omega)+1})+(1-\epsilon)(1-\beta)\beta^{L(p_{01},\omega)+1}\mathcal{T}^{L(p_{01},\omega)}(p_{01})},\\
	C_{2}=\frac{(1-\epsilon)\beta^{L(p_{01},\omega)}\mathcal{T}^{L(p_{01},\omega)}(p_{01})}{(1-\beta(1-\epsilon)p_{11})(1-\beta^{L(p_{01},\omega)+1})+(1-\epsilon)(1-\beta)\beta^{L(p_{01},\omega)+1}\mathcal{T}^{L(p_{01},\omega)}(p_{01})}.
\end{aligned}
\]
Similarly, when $p_{01}>p_{11}$, we have
\begin{equation*}
	V_{\beta,m}(p_{11})=\begin{cases}
		\frac{(1-\epsilon)(p_{11}(1-\beta)+\beta
			p_{01})}{(1-\beta)(1-\beta(1-\epsilon)p_{11}+\beta(1-\epsilon)p_{01})},
		\ \text{if }\omega_{\beta}^{*}(m)<p_{11}\\[8pt]
		\frac{(1-\beta(1-(1-\epsilon)p_{01}))m+\beta(1-\epsilon)\mathcal{T}(p_{11})(1-\beta)+\beta^{2}(1-\epsilon)p_{01}}{(1-\beta)(1+\beta(1+\beta)(1-\epsilon)p_{01}-\beta^{2}(1-\epsilon)\mathcal{T}(p_{11}))},
		\ \\ ~~~~~~~~~~~~~~~\text{if }p_{11}\leq\omega_{\beta}^{*}(m)<\mathcal{T}(p_{11})\\[6pt]
		\frac{m}{1-\beta}, \ \text{if }\omega_{\beta}^{*}(m)\geq\mathcal{T}(p_{11})
	\end{cases},
\end{equation*}
\begin{equation*}
	V_{\beta,m}(p_{01})=\begin{cases}
		\frac{(1-\epsilon)p_{01}+\beta(1-\epsilon)p_{01}V_{\beta,m}(p_{11})}{1-\beta(1-(1-\epsilon)p_{01})}, &\text{if }\omega_{\beta}^{*}(m)<p_{01}\\
		\frac{m}{1-\beta}, &\text{if }\omega_{\beta}^{*}(m)\geq p_{01}
	\end{cases}.~~~~~~~~
\end{equation*}
The approximate Whittle index is given by
\begin{equation*}
	W(\omega)=\begin{cases}
		\frac{\omega(1-\epsilon)(1-\beta p_{11}+\beta p_{01})}{1-\beta(1-\epsilon)p_{11}+\beta(1-\epsilon)p_{01}}, &\text{if }\omega\leq p_{11}\\[6pt]
		\frac{(1-\epsilon)(1-\beta+C_{4}\beta)(\beta p_{01}+\omega-\beta\mathcal{T}(\omega))}{1-\beta(1-(1-\epsilon)p_{01})+(1-\epsilon)C_{3}\beta(\beta\mathcal{T}(\omega)-\beta p_{01}-\omega)}, &\text{if }p_{11}<\omega<\omega_{0}\\[6pt]
		\frac{(1-\epsilon)(1-\beta+\beta C_{4})(\beta p_{01}+\omega(1-\beta))}{1-\beta(1-(1-\epsilon)p_{01})-(1-\epsilon)\beta C_{3}(\beta p_{01}+\omega-\beta\omega)}, &\text{if }\omega_{0}\leq\omega<\mathcal{T}(p_{11})\\[6pt]
		\frac{(1-\epsilon)(\beta p_{01}+(1-\beta)\omega)}{1+(1-\epsilon)\beta(p_{01}-\omega)}, &\text{if }\mathcal{T}(p_{11})\leq\omega<p_{01}\\[6pt]
		(1-\epsilon)\omega, &\text{if }\omega\geq p_{01}
	\end{cases},
\end{equation*}
where
\[
\begin{aligned}
	C_{3}=\tfrac{1-\beta(1-(1-\epsilon)p_{01})}{1+\beta(1+\beta)(1-\epsilon)p_{01}-\beta^{2}(1-\epsilon)\mathcal{T}(p_{11})},\\
	C_{4}=\tfrac{\beta(1-\epsilon)\mathcal{T}(p_{11})(1-\beta)+\beta^{2}(1-\epsilon)p_{01}}{1+\beta(1+\beta)(1-\epsilon)p_{01}-\beta^{2}(1-\epsilon)\mathcal{T}(p_{11})}.
\end{aligned}
\]

% endregion Main results
%%%%%%%%%%%%%%%%%%%%%%%%%%%%%%%%%%%%%%%%%%
% region Algorithm
\subsection{Algorithm}\label{sec:alg}

Our analysis leads to the algorithm for the RMAB model with imperfect observations in
Algorithm~\ref{alg:whittle} for all $\beta\in(0,1)$.
\begin{algorithm}[H]
	\caption{Whittle Index Policy}
	\hspace*{0.02in} {\bf Input:}
	$\beta\in(0,1)$, $T\ge1$,$N\ge2$, $1\le M<N$, iteration number $k$\\
	\hspace*{0.02in} {\bf Input:}
	initial belief state $\omega_n(1)$, $\textbf{P}^{(n)}$, $B_n$, $n=1,...,N$
	\begin{algorithmic}[1]
		\For{$t=1,2,...,T$}
		\For{$n=1,..., N$}
		\State Set the threshold $\omega^*_\beta(m)=\omega_n(t)$ in~\eqref{eqn:vExpand}
		\State Compute $L\left(p_{11}^{(n)}, \omega_n(t)\right)$ and set $\omega=p_{11}^{(n)}$ in~\eqref{eqn:vExpand}
		\State Expand~\eqref{eqn:vExpand} to the $k$th step and solve for $V^{(n)}_{\beta,m}(p_{11}^{(n)})$
		\State Compute $L\left(\Tau\circ\phi(\omega), \omega_n(t)\right)$ and set $\omega=\Tau\circ\phi(\omega)$ in~\eqref{eqn:vExpand}
		\State Expand~\eqref{eqn:vExpand} to the $k$th step and solve for $V^{(n)}_{\beta,m}(\Tau\circ\phi(\omega))$ from $V^{(n)}_{\beta,m}(p_{11}^{(n)})$
		\State Compute $L\left(\Tau(\omega), \omega_n(t)\right)$ and set $\omega=\Tau(\omega)$ in~\eqref{eqn:vExpand}
		\State Expand~\eqref{eqn:vExpand} to the $k$th step and solve for $V^{(n)}_{\beta,m}(\Tau(\omega))$ from $V^{(n)}_{\beta,m}(p_{11}^{(n)})$
		\State Solve for $V^{(n)}_{\beta,m}(\omega;u=1)$ by $V^{(n)}_{\beta,m}(p_{11}^{(n)})$ and $V^{(n)}_{\beta,m}(\Tau\circ\phi(\omega))$ as in~\eqref{eqn:v1Expand}
		\State Solve for $V^{(n)}_{\beta,m}(\omega;u=0)$ by $V^{(n)}_{\beta,m}(\Tau(\omega))$ as in~\eqref{eqn:v0Expand}
		\State Evaluate the solvability of the linear equation of $m$: $V^{(n)}_{\beta,m}(\omega;u=1)$
		\Statex  \qquad\quad~$=V^{(n)}_{\beta,m}(\omega;u=0)$
		\State Set $W(\omega_n(t))=\omega_n(t) B_n$ and skip Step 14 if the above is unsolvable
		\State Compute $W(\omega_n(t))$ as the solution to $V^{(n)}_{\beta,m}(\omega;u=1)=V^{(n)}_{\beta,m}(\omega;u=0)$
		\EndFor
		\State Choose the top $M$ arms with the largest Whittle Indices $W(\omega_n(t))$
		\State Observe the selected $M$ arms and accrue reward $O_n(t)S_n(t)B_n$ from each
		\State observed arm
		\For{$n=1,...,N$}
		\State Update the belief state $\omega_n(t)$ according to \eqref{eqn:beliefUpdate}
		\EndFor
		\EndFor
	\end{algorithmic}\label{alg:whittle}
\end{algorithm}

% endregion Algorithm
%%%%%%%%%%%%%%%%%%%%%%%%%%%%%%%%%%%%%%%%%%
% region Optimality for Homogeneous Systems
\section{Optimality for Homogeneous Systems}\label{sec:optimality}

A space-wise homogeneous system for a restless bandit is defined as
the system with $N$ stochastically identical arms, \ie the parameters
$\textbf{P}^{(n)}$ and $B_n$ do not depend on~$n$. In this case, our
algorithm is equivalent to the myopic policy that chooses the arms
with the largest belief values and is optimal.
\begin{theorem}
	Consider a space-wise homogeneous model with positively correlated arms ($p_{11}\ge p_{01}$) and $\epsilon$ satisfying
	\begin{eqnarray}\label{myopic:structCondition}
		\epsilon \leq \frac{p_{01}(1-p_{11})}{p_{11}(1-p_{01})}=\frac{p_{01}p_{10}}{p_{11}p_{00}},
	\end{eqnarray}
	the myopic policy is optimal over both finite and infinite horizons.
\end{theorem}
\begin{proof}
	We adopt notations similar to that in~\citet{LWZ2011} for the case of
	perfect observation ($\epsilon=\delta=0$) but need several
	non-trivial differences due to the additional complexity introduced
	by observation errors. Consider $N$ arms in total and we choose $K$
	arms to active at each step. Let
	$W_{s}(\omega_{1},\ldots,\omega_{N})$ denote the expected total
	discounted reward over $s$ steps when all arms are ordered so the
	probabilities that the underlying random processes are in state 1
	are $\omega_{1}\ge\cdots\ge\omega_{N}$. In \citet{LZK2010}, it has
	been proved that the myopic policy has a dynamic queuing structure
	if the error probability $\epsilon$
	satisfies~\eqref{myopic:structCondition}. Then we have
	{\begin{eqnarray*}
			&&W_{s+1}(\omega_{1},\ldots,\omega_{N})=(1-\epsilon)\sum_{i=1}^{K}\omega_{i} \\
			&& ~~~~~~~~~~~~~~~~~ +\beta
			\mathbbm{E}[W_{s}(p_{11},\ldots,p_{11},\tau(\omega_{K+1}),\ldots,\tau(\omega_{N}),
			\sigma(\cdot),\cdots,\sigma(\cdot))],
	\end{eqnarray*}}
	where
	$W_{0}(\cdot)=0,\tau(\omega)=p_{11}\omega+p_{01}(1-\omega),\sigma(\omega)=\tau(\frac{\epsilon\omega}{\epsilon\omega+1-\omega}),$
	and the expectation is taken over possible outcomes that can occur
	when the $K$ arms that are observed are those at the left end
	($i.e.,$ having probabilities $\omega_{1},\ldots,\omega_{K}$ that
	the underlying random processes are in state 1). We will describe
	it more specifically. For a belief state sequence
	$(\omega_{1},\ldots,\omega_{K})$, we call
	$(\omega_{i_{1}},\ldots,\omega_{i_{s}})$ and
	$(\omega_{j_{1}},\ldots,\omega_{j_{t}})$ a partition of
	$(\omega_{1},\ldots,\omega_{K})$ if they satisfy:
	\begin{itemize}
		\item[(i)] $(\omega_{i_{1}},\ldots,\omega_{i_{s}},\omega_{j_{1}},\ldots,\omega_{j_{t}})$ is a rearrangement of $(\omega_{1},\ldots,\omega_{K})$;
		\item[(ii)] $i_{1}<\cdots<i_{s}$ and $j_{1}<\cdots<j_{t}$.
	\end{itemize}
	Let $\mathcal{P}$ be the collection of all partitions of
	$(\omega_{1},\ldots,\omega_{K})$, the expectation above can be written
	as follows:%
	\normalsize
	\begin{eqnarray*}
		&&~\mathbbm{E}[W_{s}(p_{11},\ldots,p_{11},\tau(\omega_{K+1}),\ldots,\tau(\omega_{N}),
		\sigma(\cdot),\ldots,\sigma(\cdot))] \notag \\
		= &&\sum_{\mathcal{P}}(\prod_{m=1}^{s}(1-\epsilon)\omega_{i_{m}})(\prod_{n=1}^{t}(1-(1-\epsilon)\omega_{j_{n}})) \\
		&&\times
		W_{s}(p_{11},\ldots,p_{11},\tau(\omega_{K+1}),\ldots,\tau(\omega_{N}),\sigma(\omega_{j_{1}}),\ldots,\sigma(\omega_{j_{t}})).
	\end{eqnarray*}
	We can see that when
	$\omega_{1}\geq\omega_{2}\geq\cdots\geq\omega_{N}$,
	$W_{s}(\omega_{1},\ldots,\omega_{N})$ is the value function for the
	myopic policy.
	
	To prove the theorem, we first prove that $\forall s$,
	$W_{s}(\omega_{1},\ldots,\omega_{N})$ is linear in
	$\omega_{i}(1\leq i\leq N)$. We will prove it by induction. It is
	obvious that
	$W_{0}=0,
	W_{1}(\omega_{1},\ldots,\omega_{N})=(1-\epsilon)\sum_{i=1}^{K}\omega_{i}$
	are linear in $\omega_{1},\ldots,\omega_{N}$. Assume it is true for
	$s$,
	$i.e.,
	W_{s}(\omega_{l})=W_{s}(\omega_{1},\ldots,\omega_{l},\ldots,\omega_{N})=a_{l}\omega_{l}+b_{l}$
	$(\forall 1\leq l\leq N)$, where $a_{l}$ and $b_{l}$ are constants
	independent of $\omega_{l}$. Now consider
	$W_{s+1}(\omega_{l})=W_{s+1}(\omega_{1},\ldots,\omega_{l},\ldots,\omega_{N})$. When
	$l>K,$
	\[
	{\small \begin{aligned}
			W_{s+1}(\omega_{1},\ldots,&\omega_{l},\cdots,\omega_{N})=(1-\epsilon)\sum_{i=1}^{K}\omega_{i}\\
			&+\beta
			\mathbbm{E}[W_{s}(p_{11},\ldots,p_{11},\tau(\omega_{K+1}),\ldots,\tau(\omega_{l}),\ldots,\tau(\omega_{N}),\sigma(\cdot),\ldots,\sigma(\cdot))]~~~~~~~~
	\end{aligned}}
	\]
	\normalsize
	In this case, probability terms in expectation are only related to $\omega_{1},\ldots,\omega_{K}$, $\tau(\omega_{l})$ is linear in $\omega_{l}$ and $W_{s}$ is linear in $\tau(\omega_{l})$, thus $W_{s+1}$ is also linear in $\omega_{l}$. \\
	When $l\leq k$, for any partition
	$(\omega_{i_{1}},\ldots,\omega_{i_{s}})$ and
	$(\omega_{j_{1}},\ldots,\omega_{j_{t}})$ of
	$(\omega_{1},\ldots,\omega_{K})$, if
	$\omega_{l}\in\{\omega_{i_{1}},\ldots,\omega_{i_{s}}\}$, the
	corresponding term
	\begin{align*}
		(\prod_{m=1}^{s}(1-\epsilon)&\omega_{i_{m}})\cdot(\prod_{n=1}^{t}(1-(1-\epsilon)\omega_{j_{n}}))\\&~~~~\times W_{s}(p_{11},\ldots,p_{11},\tau(\omega_{K+1}),\ldots,\tau(\omega_{N}),\sigma(\omega_{j_{1}}),\ldots,\sigma(\omega_{j_{t}}))~~~~~~~~~~~~~~~~~~~
	\end{align*}
	\normalsize in the expectation is linear in $\omega_{l}$. If
	$\omega_{l}\in\{\omega_{j_{1}},\ldots,\omega_{j_{t}}\}$, by inductive
	hypothesis there exists $\tilde{a},~\tilde{b}$,\
	\begin{align*}
		&(1-(1-\epsilon)\omega_{l})W_{s}(p_{11},\ldots,p_{11},\tau(\omega_{K+1}),\ldots,\tau(\omega_{N}),\sigma(\omega_{j_{1}}),\ldots,\sigma(\omega_{l}),\ldots,\sigma(\omega_{j_{t}})) \notag \\
		&=(1-(1-\epsilon)\omega_{l})(\tilde{a}\sigma(\omega_{l})+\tilde{b})=\tilde{a}(\epsilon
		p_{11}\omega_{l}+p_{01}(1-\omega_{l}))+\tilde{b}(1-(1-\epsilon)\omega_{l}).
	\end{align*}
	\normalsize
	The equation above shows that $W_{s+1}$ is linear in $\omega_{l}$, thus the proposition is proved.
	\\ \hspace*{0.8cm}From above, we can assume
	$W_{s}(\omega_{1},\ldots,x,\ldots,y,\ldots,\omega_{N})=ax+by+cxy+d$,
	where $a,b,c,d$ are constants. If we swap the positions of $x$ and $y$
	and make differences between the two, we have
	\begin{align*}
		&W_{s}(\omega_{1},\ldots,x,\ldots,y,\ldots,\omega_{N})-W_{s}(\omega_{1},\ldots,y,\ldots,x,\ldots,\omega_{N}) \notag \\
		=&(x-y)[W_{s}(\omega_{1},\ldots,1,\ldots,0,\ldots,\omega_{N})-W_{s}(\omega_{1},\ldots,0,\ldots,1,\ldots,\omega_{N})].
	\end{align*}
	Next we will prove two important properties of $W_{s}$. We let
	$\bar{\omega_{i}}$ denote any sequence of $\omega_{i}$s, possibly
	empty.
	We still adopt induction to prove next two properties:\\[1em]
	(A) $1-\epsilon+W_{s}(\bar{\omega_{1}},y,\bar{\omega_{2}},\bar{\omega_{3}})-W_{s}(\bar{\omega_{1}},\bar{\omega_{2}},y,\bar{\omega_{3}})\geq 0.$ \\
	(B) $\forall y>x, W_{s}(\bar{\omega_{1}},y,\bar{\omega_{2}},x,\bar{\omega_{3}})-W_{s}(\bar{\omega_{1}},x,\bar{\omega_{2}},y,\bar{\omega_{3}})\geq 0.$ \\[1em]
	These are clearly true for $s=1$. We will begin by proving an
	induction step for (B). As above, the expression in (B) is equal to
	$
	(y-x)[W_{s}(\bar{\omega_{1}},1,\bar{\omega_{2}},0,\bar{\omega_{3}})-W_{s}(\bar{\omega_{1}},0,\bar{\omega_{2}},1,\bar{\omega_{3}})].
	$
	\\ Suppose the position exchange occur in the $i$th and $j$th place,
	$i<j$. If $i,j\leq K$, for some
	$\bar{\omega_{1}}',\bar{\omega_{2}}',\bar{\omega_{3}}'$(which are
	stochastically determined by the observations from the top $K$ arms in
	the queue), by inductive hypothesis,
	\begin{eqnarray*}
		&&~~~~W_{s}(\bar{\omega_{1}},1,\bar{\omega_{2}},0,\bar{\omega_{3}})-W_{s}(\bar{\omega_{1}},0,\bar{\omega_{2}},1,\bar{\omega_{3}})\\
		&&={}  \beta
		\mathbbm{E}[W_{s-1}(\bar{\omega_{1}}',\bar{\omega_{2}}',p_{01},\bar{\omega_{3}}')-W_{s-1}(\bar{\omega_{1}}',p_{01},\bar{\omega_{2}}',\bar{\omega_{3}}')]
		\geq {}  0.
	\end{eqnarray*} Similarly if $i,j>K$,
	\begin{eqnarray*}
		&&~~~~W_{s}(\bar{\omega_{1}},1,\bar{\omega_{2}},0,\bar{\omega_{3}})-W_{s}(\bar{\omega_{1}},0,\bar{\omega_{2}},1,\bar{\omega_{3}})\\
		&&= \beta
		\mathbbm{E}[W_{s-1}(\bar{\omega_{1}}',p_{11},\bar{\omega_{2}}',p_{01},\bar{\omega_{3}}')-W_{s-1}(\bar{\omega_{1}}',p_{01},\bar{\omega_{2}}',p_{11},\bar{\omega_{3}}')] \geq 0.
	\end{eqnarray*}
	The interesting case is $i\leq K<j$. In this case,
	\begin{align*} &
		W_{s}(\bar{\omega_{1}},1,\bar{\omega_{2}},0,\bar{\omega_{3}})-W_{s}(\bar{\omega_{1}},0,\bar{\omega_{2}},1,\bar{\omega_{3}})
		\notag \\ ={}& 1-\epsilon+\beta
		\mathbbm{E}[W_{s-1}(\bar{\omega_{1}}',p_{11},\bar{\omega_{2}}',p_{01},\bar{\omega_{3}}',\bar{\omega_{4}}')-W_{s-1}(\bar{\omega_{1}}',\bar{\omega_{2}}',p_{11},\bar{\omega_{3}}',p_{01},\bar{\omega_{4}}')]
		\notag \\ \geq {} & 1-\epsilon+\beta
		\mathbbm{E}[W_{s-1}(\bar{\omega_{1}}',\bar{\omega_{2}}',p_{11},p_{01},\bar{\omega_{3}}',\bar{\omega_{4}}')-W_{s-1}(\bar{\omega_{1}}',\bar{\omega_{2}}',p_{11},\bar{\omega_{3}}',p_{01},\bar{\omega_{4}}')]
		\notag \\ ={}& (1-\epsilon)(1-\beta)+\beta
		\mathbbm{E}[(1-\epsilon)\notag \\
		&+W_{s-1}(\bar{\omega_{1}}',\bar{\omega_{2}}',p_{11},p_{01},\bar{\omega_{3}}',\bar{\omega_{4}}')
		-W_{s-1}(\bar{\omega_{1}}',\bar{\omega_{2}}',p_{11},\bar{\omega_{3}}',p_{01},\bar{\omega_{4}}')]
		\notag \\ \geq {} & 0,
	\end{align*} \normalsize where the first inequality follows from
	the inductive hypothesis for (B) and the second follows from (A). \\
	Next we will prove (A) by induction. Suppose that $y$ occurs within
	the two expressions in the $i$th and $j$th place, $i<j$. If $i,j\leq
	K$, similarly for some
	$\bar{\omega_{1}}',\bar{\omega_{2}}',\bar{\omega_{3}}'$ (they depend
	on the observations),
	\begin{align*} &
		1-\epsilon+W_{s}(\bar{\omega_{1}},y,\bar{\omega_{2}},\bar{\omega_{3}})-W_{s}(\bar{\omega_{1}},\bar{\omega_{2}},y,\bar{\omega_{3}})
		\notag \\ ={} & 1-\epsilon+\beta
		\mathbbm{E}[W_{s-1}(\bar{\omega_{1}}',\sigma(y),\bar{\omega_{2}}',\bar{\omega_{3}}')-W_{s-1}(\bar{\omega_{1}}',\bar{\omega_{2}}',\sigma(y),\bar{\omega_{3}}')]
		\notag \\ ={} & (1-\epsilon)(1-\beta)+\beta
		\mathbbm{E}[(1-\epsilon)+W_{s-1}(\bar{\omega_{1}}',\sigma(y),\bar{\omega_{2}}',\bar{\omega_{3}}')-W_{s-1}(\bar{\omega_{1}}',\bar{\omega_{2}}',\sigma(y),\bar{\omega_{3}}')]
		\notag \\ \geq {} & 0.
	\end{align*}\normalsize If $i,j>K$, we have
	\begin{align*} &
		1-\epsilon+W_{s}(\bar{\omega_{1}},y,\bar{\omega_{2}},\bar{\omega_{3}})-W_{s}(\bar{\omega_{1}},\bar{\omega_{2}},y,\bar{\omega_{3}})
		\notag \\ = {} & (1-\epsilon)(1-\beta)+\beta
		\mathbbm{E}[(1-\epsilon)+W_{s-1}(\bar{\omega_{1}}',\tau(y),\bar{\omega_{2}}',\bar{\omega_{3}}')-W_{s-1}(\bar{\omega_{1}}',\bar{\omega_{2}}',\tau(y),\bar{\omega_{3}}')]
		\notag \\ \geq {} & 0.
	\end{align*} \normalsize The interesting case is $i\leq K<j$. Let
	$\bar{\omega_{2}}=(\bar{\omega_{21}},x,\bar{\omega_{22}})$, where
	$\bar{\omega_{1}}$ and $\bar{\omega_{21}}$ represent $K-1$ states in
	total. Then
	\begin{eqnarray*}
		&&~~~1-\epsilon+W_{s}(\bar{\omega_{1}},y,\bar{\omega_{2}},\bar{\omega_{3}})-W_{s}(\bar{\omega_{1}},\bar{\omega_{2}},y,\bar{\omega_{3}})\\
		&&={}
		1-\epsilon+W_{s}(\bar{\omega_{1}},y,\bar{\omega_{21}},x,\bar{\omega_{22}},\bar{\omega_{3}})-W_{s}(\bar{\omega_{1}},\bar{\omega_{21}},x,\bar{\omega_{22}},y,\bar{\omega_{3}}).
	\end{eqnarray*} The above expression is a function of $x$ and $y$, of the
	form $ax+by+cxy+d$. To prove the expression above is nonnegative for
	all $x,y\in[0,1]$, we just need to check out that it is true for
	$(x,y)\in \{(0,0),(0,1),(1,0),(1,1)\}$. \\ If $x=y=0$, then
	\begin{eqnarray*}
		&&~~~1-\epsilon+W_{s}(\bar{\omega_{1}},0,\bar{\omega_{21}},0,\bar{\omega_{22}},\bar{\omega_{3}})-W_{s}(\bar{\omega_{1}},\bar{\omega_{21}},0,\bar{\omega_{22}},0,\bar{\omega_{3}})
		\notag \\ &&= {}  1-\epsilon+\beta
		\mathbbm{E}[W_{s-1}(\bar{\omega_{1}}',p_{01},\tau(\bar{\omega_{22}}),\bar{\omega_{3}}',p_{01},\bar{\omega_{4}}')\notag\\
		&&\quad-W_{s-1}(\bar{\omega_{1}}',\tau(\bar{\omega_{22}}),p_{01},\bar{\omega_{3}}',\bar{\omega_{4}}',p_{01})]
		\notag \\&& \geq {}  (1-\epsilon)(1-\beta)+\beta
		\mathbbm{E}[(1-\epsilon)\notag\\
		&&\quad+W_{s-1}(\bar{\omega_{1}}',p_{01},\tau(\bar{\omega_{22}}),\bar{\omega_{3}}',p_{01},\bar{\omega_{4}}')
		-W_{s-1}(\bar{\omega_{1}}',\tau(\bar{\omega_{22}}),\bar{\omega_{3}}',p_{01},\bar{\omega_{4}}',p_{01})]
		\notag \\ &&\geq {}  0.
	\end{eqnarray*} \normalsize If $x=y=1$, then
	\begin{eqnarray*}
		&&~~~1-\epsilon+W_{s}(\bar{\omega_{1}},1,\bar{\omega_{21}},1,\bar{\omega_{22}},\bar{\omega_{3}})-W_{s}(\bar{\omega_{1}},\bar{\omega_{21}},1,\bar{\omega_{22}},1,\bar{\omega_{3}})
		\notag \\&& ={}  1-\epsilon+\beta
		\mathbbm{E}[W_{s-1}(\bar{\omega_{1}}',p_{11},\bar{\omega_{2}}',p_{11},\tau(\bar{\omega_{22}}),\bar{\omega_{3}}')\notag\\
		&&\quad-W_{s-1}(\bar{\omega_{1}}',\bar{\omega_{2}}',p_{11},\tau(\bar{\omega_{22}}),p_{11},\bar{\omega_{3}}')]
		\notag \\ &&\geq {}  (1-\epsilon)(1-\beta)+\beta
		\mathbbm{E}[(1-\epsilon)\notag\\
		&&\quad+W_{s-1}(\bar{\omega_{1}}',\bar{\omega_{2}}',p_{11},p_{11},\tau(\bar{\omega_{22}}),\bar{\omega_{3}}')
		-W_{s-1}(\bar{\omega_{1}}',\bar{\omega_{2}}',p_{11},\tau(\bar{\omega_{22}}),p_{11},\bar{\omega_{3}}')]
		\notag \\&& \geq {}  0.
	\end{eqnarray*}\normalsize If $x=0,y=1$, then
	\begin{eqnarray*}
		&&~~~1-\epsilon+W_{s}(\bar{\omega_{1}},1,\bar{\omega_{21}},0,\bar{\omega_{22}},\bar{\omega_{3}})-W_{s}(\bar{\omega_{1}},\bar{\omega_{21}},0,\bar{\omega_{22}},1,\bar{\omega_{3}})
		\notag \\&& = 2(1-\epsilon)+\beta
		\mathbbm{E}[W_{s-1}(\bar{\omega_{1}}',p_{11},\bar{\omega_{3}}',p_{01},\tau(\bar{\omega_{22}}),\bar{\omega_{4}}')
		\notag\\
		&&\quad-W_{s-1}(\bar{\omega_{1}}',\bar{\omega_{3}}',\tau(\bar{\omega_{22}}),p_{11},\bar{\omega_{4}}',p_{01})]
		\notag \\&& \geq 2(1-\epsilon)(1-\beta)+\beta
		\mathbbm{E}[2-2\epsilon\notag\\
		&&\quad+W_{s-1}(\bar{\omega_{1}}',\bar{\omega_{3}}',p_{01},\tau(\bar{\omega_{22}}),p_{11},\bar{\omega_{4}}')
		-W_{s-1}(\bar{\omega_{1}}',\bar{\omega_{3}}',\tau(\bar{\omega_{22}}),p_{11},\bar{\omega_{4}}',p_{01})]
		\notag \\&& > 0.
	\end{eqnarray*}
	\normalsize If $x=1,y=0$, then
	\begin{eqnarray*}
		&&~~~1-\epsilon+W_{s}(\bar{\omega_{1}},0,\bar{\omega_{21}},1,\bar{\omega_{22}},\bar{\omega_{3}})-W_{s}(\bar{\omega_{1}},\bar{\omega_{21}},1,\bar{\omega_{22}},0,\bar{\omega_{3}})
		\notag \\ &&= \beta
		\mathbbm{E}[W_{s-1}(\bar{\omega_{1}}',p_{11},\tau(\bar{\omega_{22}}),\bar{\omega_{3}}',p_{01},\bar{\omega_{4}}')-W_{s-1}(\bar{\omega_{1}}',p_{11},\tau(\bar{\omega_{22}}),p_{01},\bar{\omega_{3}}',\bar{\omega_{4}}')]
		\notag \\  &&\geq 0.
	\end{eqnarray*} \normalsize Thus (A) is true. In fact, (B) shows that
	the myopic policy is optimal over finite horizons. By contradiction,
	it is easy to show that the myopic policy also maximizes the expected
	total discounted reward and the expected average reward over the
	infinite horizon. Furthermore, our proof does not depend on the
	time-homogeneousness of the system so the optimality result holds even
	if the system parameters are time-varying as long as $p_{11}(t)\ge
	p_{01}(t)$ and $\epsilon$ satisfies~\eqref{myopic:structCondition}.
\end{proof}

% endregion Optimality for Homogeneous Systems

%%%%%%%%%%%%%%%%%%%%%%%%%%%%%%%%%%%%%%%%%%
% region Experimental Results and Conclusion
\section{Numerical Analysis and Conclusion}\label{sec:numericalCon}

In this section, we illustrate the near-optimality and efficiency of our approximated Whittle index policy for non-homogeneous arms through simulation examples. After the discussions and illustrations on these numerical results, we conclude this paper and propose several future research directions on relevant problems.

% region Experimental Results

\subsection{Numerical Examples}\label{sec:experiments}
We will show that the $4$-iteration approximation algorithm is sufficient to yield the same performance as the exact Whittle index policy and use it to plot the
approximated Whittle index function $W(\omega)$ for the following
parameters: in Fig.~\ref{fig:whittleN}
$p_{11}=0.2, p_{01}=0.9,\beta=0.9, \epsilon=0.1, B=1$; in
Fig.~\ref{fig:whittleP}
$p_{11}=0.6, p_{01}=0.3, \beta=0.9, \epsilon=0.1,B=1$. Note that the
monotonic increasing property of $W(\omega)$ implies the indexability
numerically while the nonlinearity of $W(\omega)$ illustrates its
difference to the myopic policy (with index $\omega B$ as a linear function in~$\omega$).

We now compare the performance of Whittle index policy with the
optimal policy which is computed by dynamic programming over a finite
horizon of length~$T$. In other words, we recursively call the following equation with terminating state $V_{1,0,\beta,m}(\cdot)=0$:
\begin{eqnarray}
	V_{1,T,\beta,m}(\omega) = \max\{V_{1,T,\beta,m}(\omega;u=1);V_{1,T,\beta,m}(\omega;u=0)\},
\end{eqnarray}where $V_{1,T,\beta,m}(\omega;u=1)$ and $V_{1,T,\beta,m}(\omega;u=0)$ are respectively given by~\eqref{eqn:finite1} and~\eqref{eqn:finite0} in terms of $V_{1,T-1,\beta,m}(\cdot)$. Clearly, the number of observed belief states grows {\em exponentially} with both the number of arms (as arms are not decoupled by any relaxation) and the time horizon~$T$ due to the tree-expansion type of the belief update given in~\eqref{eqn:beliefUpdate}. In contrast, our approximate Whittle index has a {\em linear} complexity in both~$T$ and the number of arms. In Fig.~\ref{fig:complexityT} and Fig.~\ref{fig:complexityArms}, we compare the real running times between the optimal policy and our algorithm to illustrate the efficiency of the latter. Note that the optimal policy for any finite time horizon~$T$ provides an upper bound on the total discounted reward over~$T$ achieved by the infinite-horizon optimal policy. This is because one can definitely apply the infinite-horizon optimal policy to the finite-horizon problem up to time~$T$. Henceforth, the near-optimality of our algorithm is well demonstrated by comparing to the finite-horizon optimal policy over the first $T$ steps as shown in Figures~\ref{fig:3}-\ref{fig:18} (see Table~\ref{table1} and Table~\ref{table2} for system parameters). Furthermore, all numerical experiments with randomly generated system parameters showed that setting the iteration number $k=4$ is sufficient for the Whittle indices of all arms to converge such that their rank remains the same as~$k$ increases at each time step~$t$. In other words, setting $k=4$ makes the approximated Whittle index have the same action path as the exact Whittle index, leading to the same performance. When~$k$ becomes smaller, the approximation error will be larger and cause more performance loss as shown in Fig.~\ref{fig:iterCompare}.

We also illustrate the performance of the myopic policy that
chooses the $M$ arms with the largest $\omega_n B_n$ for comparison. From Figures~\ref{fig:3}-\ref{fig:18}, we observe that
Algorithm~\ref{alg:whittle} outperforms the myopic policy. Interestingly, the Whittle index policy may have some performance loss in the middle but eventually catches up with the optimal policy as time goes. This is consistent with the conjecture that Whittle index policy is asymptotically optimal as time goes to infinity as the Lagrangian relaxation should not fundamentally alter the state and action paths of the optimal policy for the original problem from the perspective of large deviation theory~\citep{WW1990}. On the contrary, the myopic policy is unable to follow the optimal action path and never catches up! Since the myopic policy only cares about maximizing the immediate reward, its performance for $T=1$ is optimal (thus better than any other policy) because the state transitions do not matter in this case. Definitely, the myopic policy has the lowest complexity but this advantage is negligible given its significant performance loss and the efficiency of our algorithm compared to the optimal policy as shown in Fig.~\ref{fig:complexityT} and Fig.~\ref{fig:complexityArms}.

% endregion Experimental Results

% region Conclusions
\subsection{Conclusion and Future Work}\label{sec:conclusions}

In this paper, we proposed a low-complexity algorithm based on the
idea of value function approximations arisen in solving for the Whittle index policy of a class of RMAB with an infinite state space and an imperfect
observation model. By exploring and exploiting the rich mathematical structure of this RMAB model, our algorithm was designed to be implemented online to control the approximation error such that it becomes equivalent to the original Whittle index policy. Extensive numerical examples showed that our algorithm achieves a near-optimal performance with a complexity linear in the key system parameters such as the time horizon~$T$ and the number of arms. From Figures~\ref{fig:3}-\ref{fig:18}, we observe that in some instances the four-iteration policy is closer to optimality than in other instances. Unfortunately, it is still unknown how to theoretically quantify the performance gap of Whittle index policy to optimality in finite regime: only some asymptotic results were obtained for restless bandits with finite states as both the number of arms and the time horizon go to infinity under the time-average reward criterion and some conditions only verified for bandits with~$2$ or~$3$ states (\citealp{WW1990,WW1991}).

Future work includes the theoretic study of the performance loss of Whittle index policy by more in-depth analysis on the value functions to further improve our algorithm. Another research direction is to consider more complex system models such as non-Markovian state models (see, \eg \citealp{LWZ2011}) or high-dimensional state models (see, \eg \citealp{L2021}) and study the convergence patterns of the state path to approximate the value functions. Future work also includes the generalization of constructing a finite set of linear equations to solve for dynamic programming problems, \eg the simplification of non-linearity by threshold-based first crossing time, the error-control process by backward induction, and further complexity reduction by minimizing the number of linear equations required.

% endregion Conclusions

%%%%%%%%%%%%%%%%%%%%%%%%%%%%%%%%%%%%%%%%%%

\clearpage

\begin{figure}
	\begin{minipage}[h]{0.5\linewidth}
		\centering
		\includegraphics[height=5.3cm,width=5.3cm]{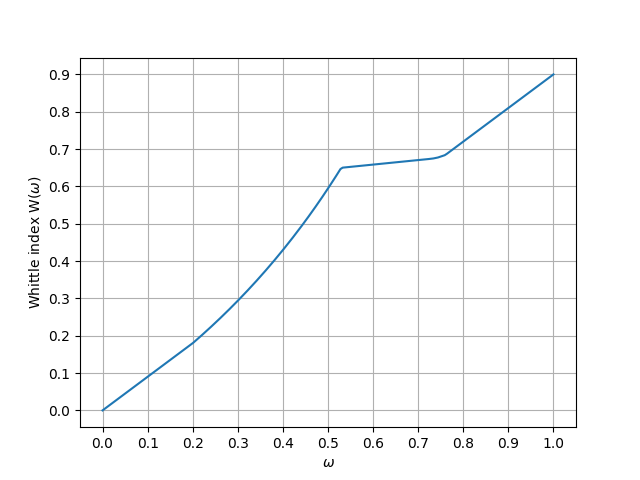}
		\caption{Approximated $W(\omega)$ ($p_{11}<p_{01}$)}
		\label{fig:whittleN}
	\end{minipage}%
	\begin{minipage}[h]{0.5\linewidth}
		\centering
		\includegraphics[height=5.3cm,width=5.3cm]{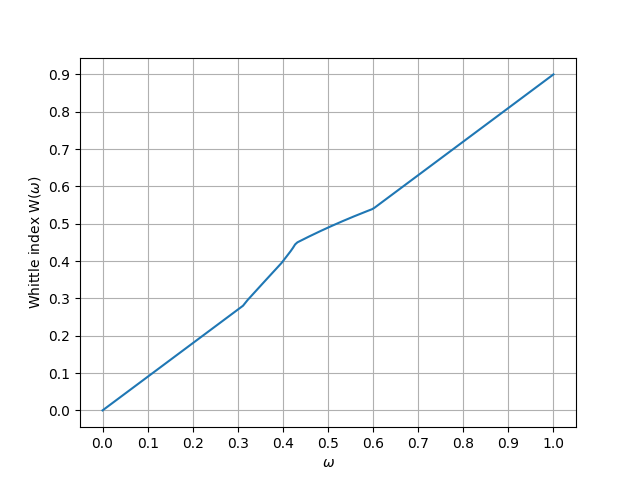}
		\caption{Approximated $W(\omega)$ ($p_{11}>p_{01}$)}
		\label{fig:whittleP}
	\end{minipage}
\end{figure}

\begin{figure}
	\begin{minipage}[h]{0.5\linewidth}
		\centering
		\includegraphics[height=5.3cm,width=5.3cm]{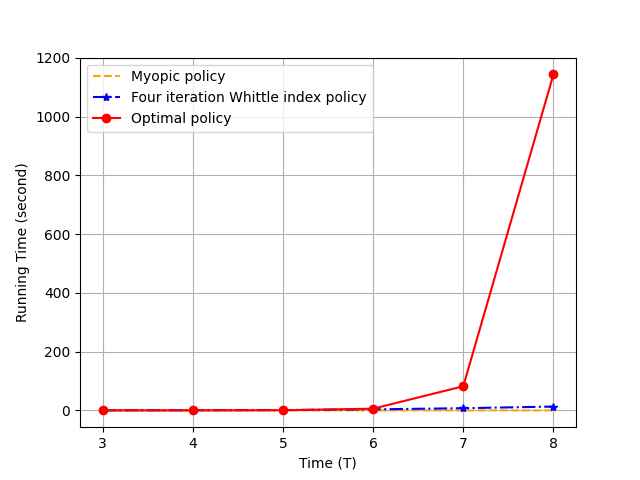}
		\caption{The Complexity with~$T$}
		\label{fig:complexityT}
	\end{minipage}%
	\begin{minipage}[h]{0.5\linewidth}
		\centering
		\includegraphics[height=5.3cm,width=5.3cm]{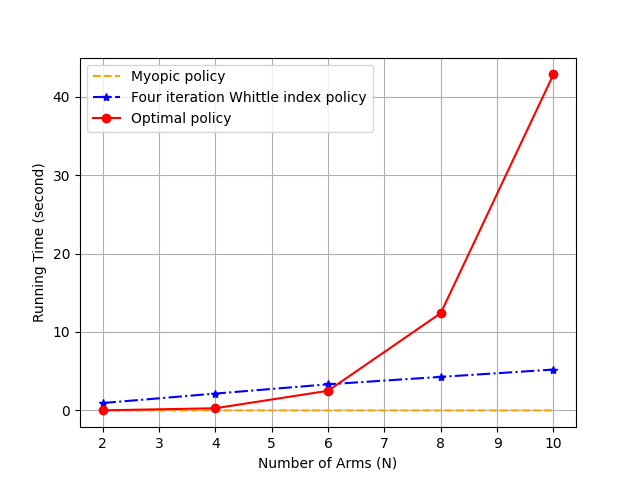}
		\caption{The Complexity with~$N$}
        \label{fig:complexityArms}
	\end{minipage}
\end{figure}

\begin{figure}
	\begin{minipage}[h]{0.5\linewidth}
		\centering
		\includegraphics[height=5.3cm,width=5.3cm]{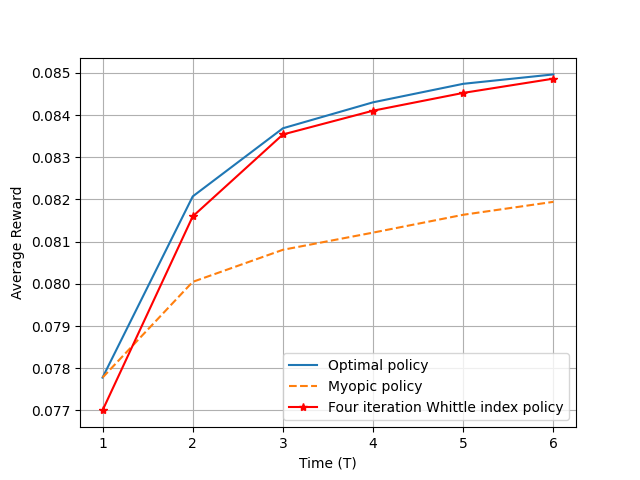}
		\caption{Example-1}
		\label{fig:3}
	\end{minipage}%
	\begin{minipage}[h]{0.5\linewidth}
		\centering
		\includegraphics[height=5.3cm,width=5.3cm]{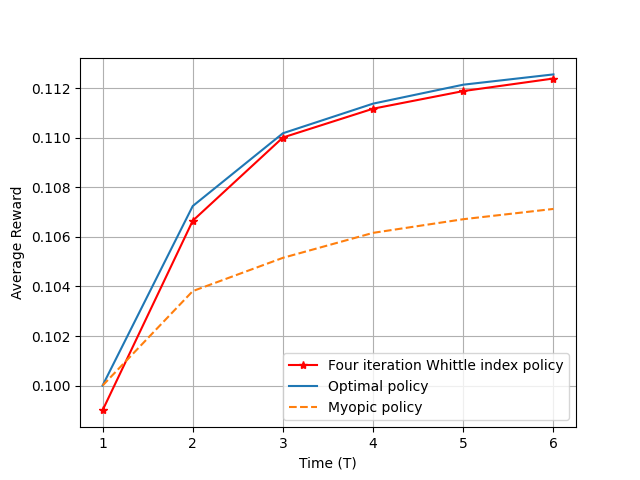}
		\caption{Example-2}
	\end{minipage}
\end{figure}

\begin{figure}
	\begin{minipage}[h]{0.5\linewidth}
		\centering
		\includegraphics[height=5.3cm,width=5.3cm]{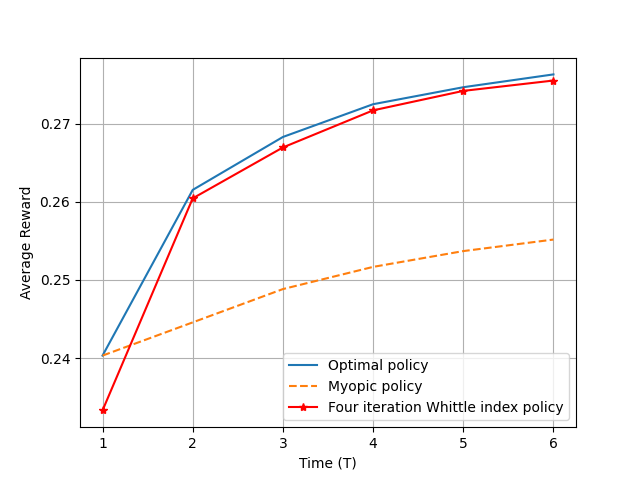}
		\caption{Example-3}
	\end{minipage}%
	\begin{minipage}[h]{0.5\linewidth}
		\centering
		\includegraphics[height=5.3cm,width=5.3cm]{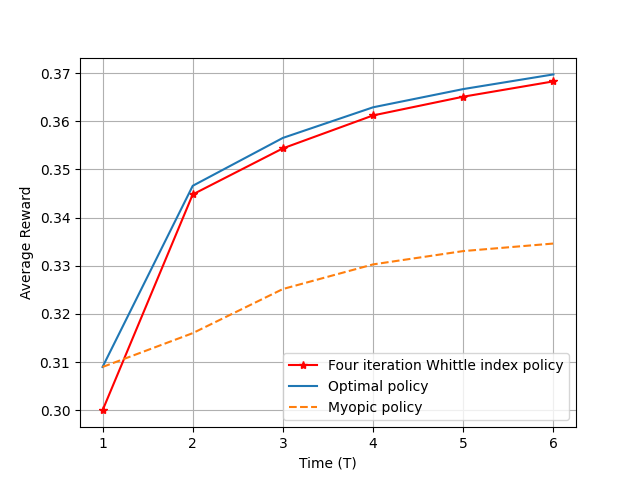}
		\caption{Example-4}
	\end{minipage}
\end{figure}

\begin{figure}
	\begin{minipage}[h]{0.5\linewidth}
		\centering
		\includegraphics[height=5.3cm,width=5.3cm]{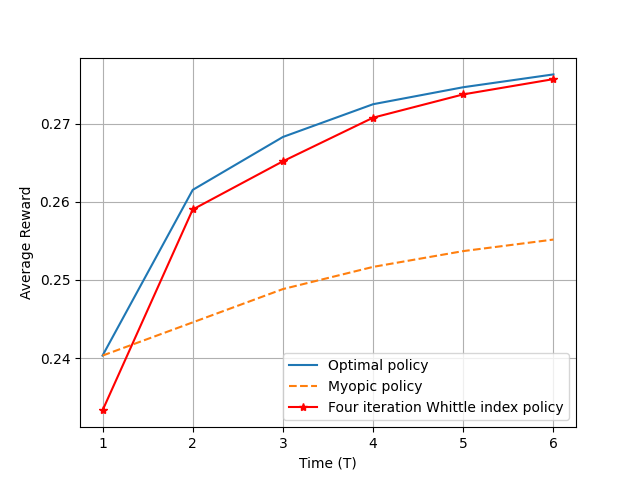}
		\caption{Example-5}
	\end{minipage}%
	\begin{minipage}[h]{0.5\linewidth}
		\centering
		\includegraphics[height=5.3cm,width=5.3cm]{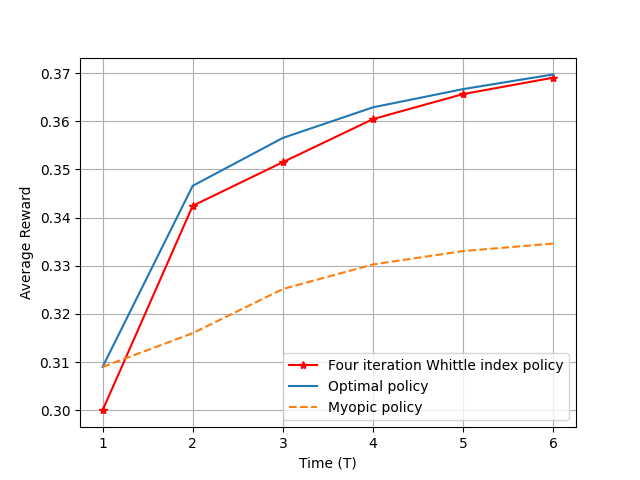}
		\caption{Example-6}
	\end{minipage}
\end{figure}

\begin{figure}
	\begin{minipage}[h]{0.5\linewidth}
		\centering
		\includegraphics[height=5.3cm,width=5.3cm]{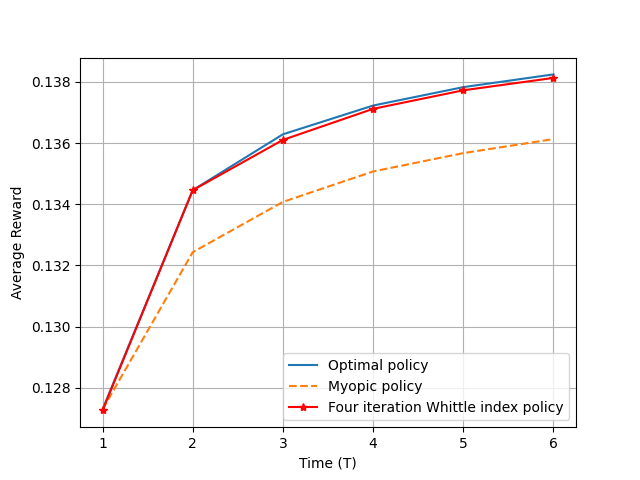}
		\caption{Example-7}
	\end{minipage}%
	\begin{minipage}[h]{0.5\linewidth}
		\centering
		\includegraphics[height=5.3cm,width=5.3cm]{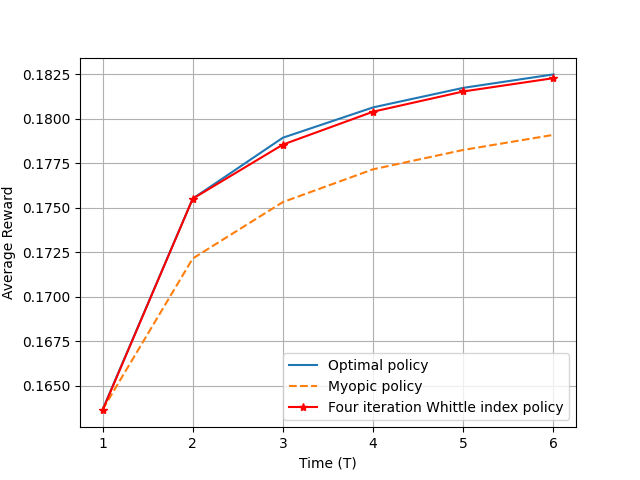}
		\caption{Example-8}
	\end{minipage}
\end{figure}

\begin{figure}
	\begin{minipage}[h]{0.5\linewidth}
		\centering
		\includegraphics[height=5.3cm,width=5.3cm]{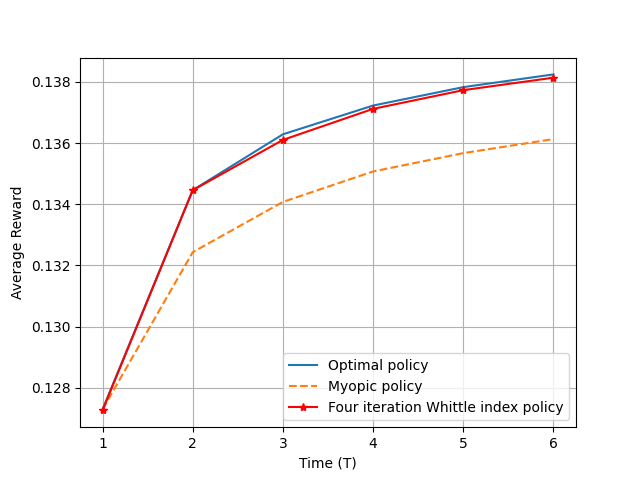}
		\caption{Example-9}
	\end{minipage}%
	\begin{minipage}[h]{0.5\linewidth}
		\centering
		\includegraphics[height=5.3cm,width=5.3cm]{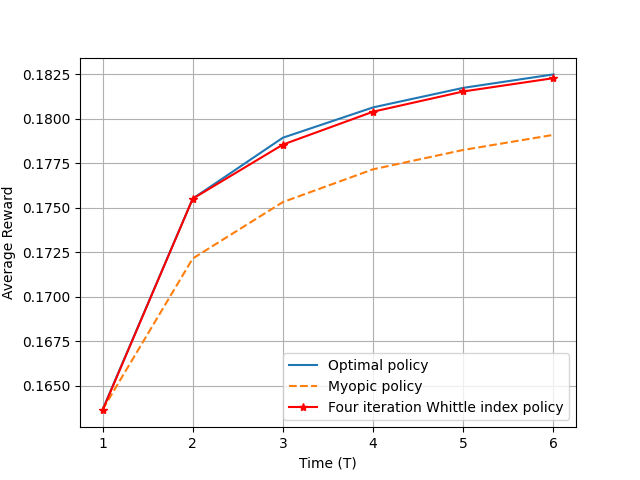}
		\caption{Example-10}
	\end{minipage}
\end{figure}

\begin{figure}
	\begin{minipage}[h]{0.5\linewidth}
		\centering
		\includegraphics[height=5.3cm,width=5.3cm]{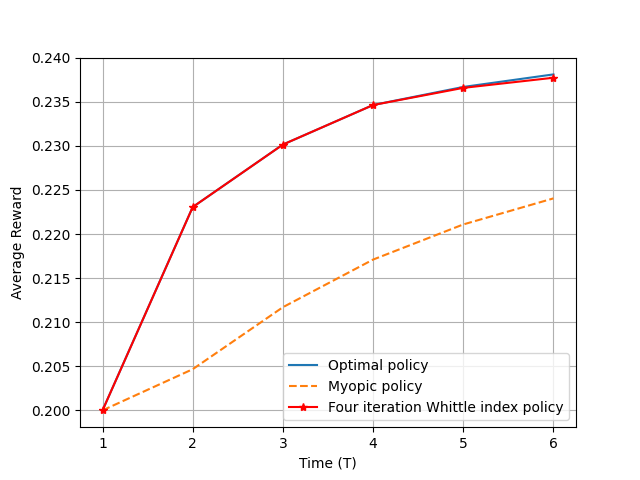}
		\caption{Example-11}
	\end{minipage}%
	\begin{minipage}[h]{0.5\linewidth}
		\centering
		\includegraphics[height=5.3cm,width=5.3cm]{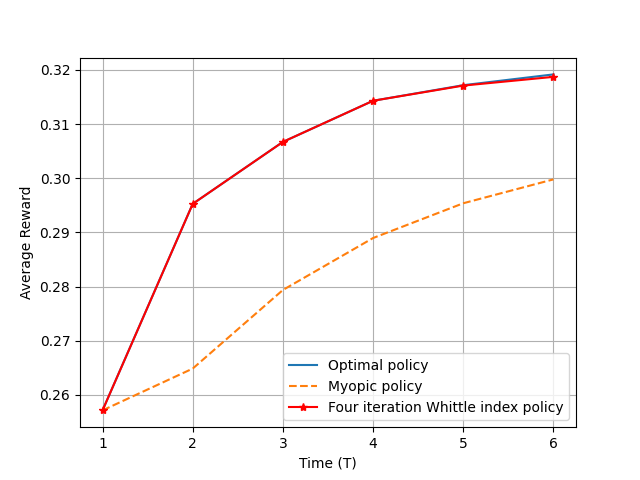}
		\caption{Example-12}
	\end{minipage}
\end{figure}

\begin{figure}
	\begin{minipage}[h]{0.5\linewidth}
		\centering
		\includegraphics[height=5.3cm,width=5.3cm]{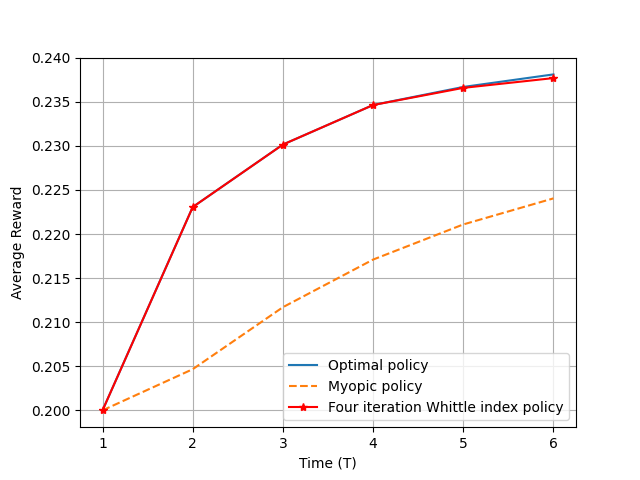}
		\caption{Example-13}
	\end{minipage}%
	\begin{minipage}[h]{0.5\linewidth}
		\centering
		\includegraphics[height=5.3cm,width=5.3cm]{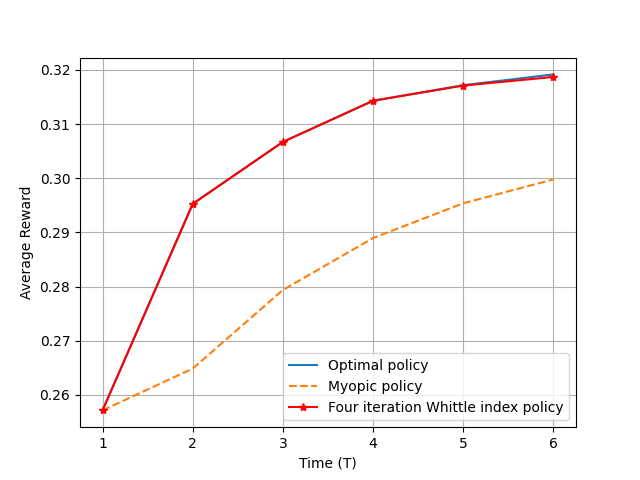}
		\caption{Example-14}
	\end{minipage}
\end{figure}

\begin{figure}
	\begin{minipage}[h]{0.5\linewidth}
		\centering
		\includegraphics[height=5.3cm,width=5.3cm]{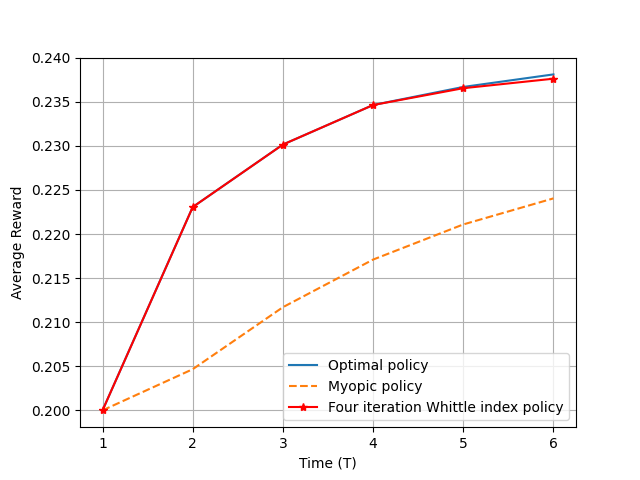}
		\caption{Example-15}
	\end{minipage}%
	\begin{minipage}[h]{0.5\linewidth}
		\centering
		\includegraphics[height=5.3cm,width=5.3cm]{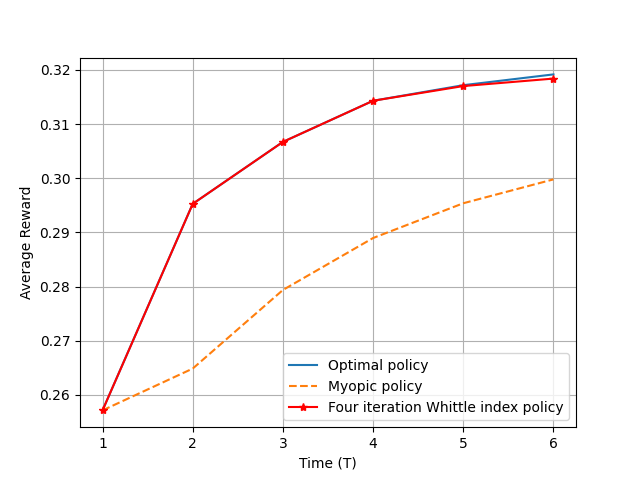}
		\caption{Example-16}
		\label{fig:18}
	\end{minipage}
\end{figure}

\begin{figure}
		\centering
		\includegraphics[height=9cm,width=9cm]{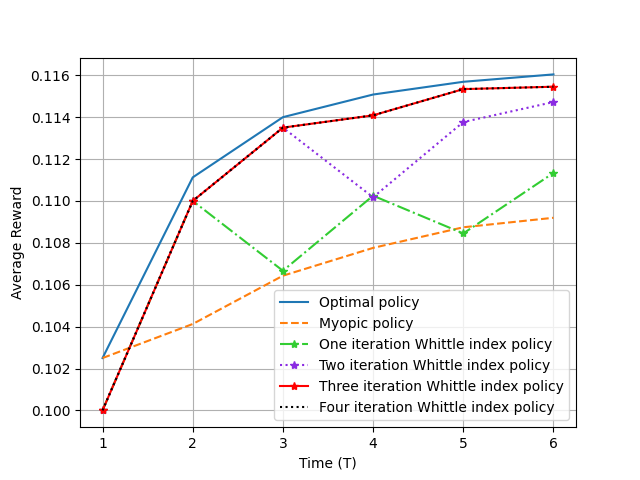}
		\caption{The Performance Comparison for~$k$}
		\label{fig:iterCompare}
\end{figure}

\clearpage
% \newpage

\begin{table}[htbp]
	\linespread{2}
    \caption{Experiment Setting}
	\centering
	\renewcommand\arraystretch{1}
    \let\cline\clineorig
	\begin{tabular}{|c|c|c|}
		\hline
		\multirowcell{3}{System-1} &
		$\{p_{11}^{(i)}\}_{i=1}^{7}$ & $\{0.3, 0.6, 0.4, 0.7, 0.2, 0.6, 0.8\}$ \\ \cline{2-3} &
		$\{p_{01}^{(i)}\}_{i=1}^{7}$ & $\{0.1, 0.4, 0.3, 0.4, 0.1, 0.3, 0.5\}$ \\ \cline{2-3} &
		$\{B_{i}\}_{i=1}^{7}$ & $\{0.8800, 0.2200, 0.3300, 0.1930, 1.0000, 0.2558, 0.1549\}$ \\
		\hline
		\multirowcell{3}{System-2} &
		$\{p_{11}^{(i)}\}_{i=1}^{7}$ & $\{0.6, 0.4, 0.2, 0.2, 0.4, 0.1, 0.3\}$ \\ \cline{2-3} &
		$\{p_{01}^{(i)}\}_{i=1}^{7}$ & $\{0.8, 0.6, 0.4, 0.9, 0.8, 0.6, 0.7\}$ \\ \cline{2-3} &
		$\{B_{i}\}_{i=1}^{7}$ & $\{0.5150, 0.6666, 1.0000, 0.6296, 0.5833, 0.8100, 0.6700\}$ \\
		\hline
		\multirowcell{3}{System-3} &
		$\{p_{11}^{(i)}\}_{i=1}^{7}$ & $\{0.1, 0.4, 0.3, 0.4, 0.1, 0.3, 0.5\}$ \\ \cline{2-3} &
		$\{p_{01}^{(i)}\}_{i=1}^{7}$ & $\{0.3, 0.6, 0.4, 0.7, 0.2, 0.6, 0.8\}$ \\ \cline{2-3} &
		$\{B_{i}\}_{i=1}^{7}$ & $\{0.7273, 0.3636, 0.5000, 0.3377, 1.0000, 0.3939, 0.2955\}$ \\
		\hline
		\multirowcell{3}{System-4} &
		$\{p_{11}^{(i)}\}_{i=1}^{7}$ & $\{0.6, 0.7, 0.2, 0.6, 0.4, 0.5, 0.3\}$ \\ \cline{2-3} &
		$\{p_{01}^{(i)}\}_{i=1}^{7}$ & $\{0.8, 0.4, 0.9, 0.5, 0.7, 0.2, 0.6\}$ \\ \cline{2-3} &
		$\{B_{i}\}_{i=1}^{7}$ & $\{0.4286, 0.5000, 0.5397, 0.5143, 0.5306, 1.0000, 0.6190\}$ \\
		\hline
	\end{tabular}
	\label{table1}
\end{table}

%\clearpage

\begin{table}[htbp]
	\linespread{2}
    \caption{Experiment Setting (continued)}
	\centering
	\renewcommand\arraystretch{1}
    \let\cline\clineorig
	\begin{tabular}{|c|c|c|c|c|c|}
		\hline
		System & Example & $\epsilon$ & $\beta$ & Meet threshold & Meet indexability\\
		~ & ~ & ~ & ~ & conditions? & conditions?\\ \hline
		\multirowcell{2}{System-1} & 1 & 0.3 & 0.999 & yes & no \\ \cline{2-6} & 2 & 0.1 & 0.999 & yes & no \\
		\hline
		\multirowcell{4}{System-2} & 3 & 0.3 & 0.29 & yes & yes \\ \cline{2-6} & 4 & 0.1 & 0.29 & yes & yes \\ \cline{2-6} & 5 & 0.3 & 0.48 & no & yes \\ \cline{2-6} & 6 & 0.1 & 0.48 & no & yes \\
		\hline
		\multirowcell{4}{System-3} & 7 & 0.3 & 0.69 & yes & no \\ \cline{2-6} & 8 & 0.1 & 0.69 & yes & no \\ \cline{2-6} & 9 & 0.3 & 0.48 & yes & yes \\ \cline{2-6} & 10 & 0.1 & 0.48 & yes & yes \\
		\hline
		\multirowcell{6}{System-4} & 11 & 0.3 & 0.29 & yes & yes \\ \cline{2-6} & 12 & 0.1 & 0.29 & yes & yes \\ \cline{2-6} & 13 & 0.3 & 0.48 & no & yes \\ \cline{2-6} & 14 & 0.1 & 0.48 & no & yes \\ \cline{2-6} & 15 & 0.3 & 0.999 & no & no \\ \cline{2-6} & 16 & 0.1 & 0.999 & no & no \\
		\hline
	\end{tabular}
	\label{table2}
\end{table}

\section*{Declarations}
\begin{itemize}
    \item {\bf Funding} Not applicable
    \item {\bf Conflict of interest/Competing interests} Not applicable
    \item {\bf Ethics approval} Not applicable
    \item {\bf Consent to participate} Not applicable
    \item {\bf Consent for publication} Yes
    \item {\bf Availability of data and materials} Available upon request
    \item {\bf Code availability} Available upon request
    \item {\bf Authors' contributions} Keqin Liu constructed the proof sketch for each theorem and the main algorithm and contributed to the writing of the paper. Richard Weber outlined the proof strategy for the optimality of the myopic policy in homogeneous systems and contributed to the verification and writing of the paper. Chengzhong Zhang filled out the details of the proofs and conducted the numerical simulations.
\end{itemize}


% \bibliography{bibliography}% common bib file

\end{document}